%% file: sparsification_paper_revision.tex
\documentclass[journal]{IEEEtran}
\IEEEoverridecommandlockouts
\usepackage{cite}
\usepackage{comment}
\usepackage{amsmath,amssymb,amsfonts,amsthm}
\allowdisplaybreaks[4]
\usepackage{graphicx}
\usepackage{caption}
\usepackage{epstopdf}
\usepackage{subcaption}
\usepackage[flushleft]{threeparttable}
\usepackage{textcomp}
\usepackage{xcolor}
\usepackage{bm,bbm}
\usepackage{algorithm}
\usepackage{setspace}
\usepackage{algpseudocode}
\usepackage{footnote}
\usepackage{url}
\usepackage{pifont}
\algrenewcommand\algorithmicrequire{\textbf{Input:}}
\algrenewcommand\algorithmicensure{\textbf{Initialize:}}
\newtheorem{assumption}{Assumption}
\newtheorem{lemma}{Lemma}
\newtheorem{theorem}{Theorem}

\def\BibTeX{{\rm B\kern-.05em{\sc i\kern-.025em b}\kern-.08em
    T\kern-.1667em\lower.7ex\hbox{E}\kern-.125emX}}
\ifCLASSINFOpdf
\else
\fi

\begin{document}
%
\title{{Distributed} Learning With Sparsified Gradient Differences\thanks{The work by Yicheng Chen and Rick S. Blum is supported by the U. S. Army Research Laboratory and the U. S. Army Research Office under grant number
W911NF-17-1-0331, the National Science
Foundation under Grant ECCS-1744129, and a grant from the Commonwealth of Pennsylvania, Department of Community and Economic Development, through the Pennsylvania Infrastructure Technology Alliance (PITA). The work by Martin Tak\'a\v{c} is partially supported by the U.S. National Science Foundation, under award number CCF-1618717 and CCF-1740796.}}

\author{Yicheng Chen, Rick S. Blum,
\IEEEmembership{Fellow,~IEEE}, Martin Tak\'a\v{c}, and Brian M. Sadler, \IEEEmembership{Life Fellow,~IEEE}
\thanks{
Yicheng Chen, Rick S. Blum, and Martin Tak\'a\v{c}  are with Lehigh University, Bethlehem, PA 18015 USA
(email: yic917@lehigh.edu, rblum@eecs.lehigh.edu, takac@lehigh.edu).}
\thanks{
Brian M. Sadler is with the Army Research Laboratory, Adelphi, MD 20783 USA (email: brian.m.sadler6.civ@army.mil).}}
\maketitle

\begin{abstract}
A very large number of communications are typically required to solve distributed learning tasks, {and this critically limits scalability and convergence speed in wireless communications applications.}
In this paper, we devise a \textcolor[rgb]{0.00,0.00,0.00}{\textbf{G}radient \textbf{D}escent method with \textbf{S}parsification and \textbf{E}rror \textbf{C}orrection (\textbf{GD-SEC})} to improve the communications efficiency
{in a general worker-server architecture}.
Motivated by a variety of wireless communications learning scenarios,
\textcolor[rgb]{0.00,0.00,0.00}{{GD-SEC} reduces the number of bits per communication from worker to server with no degradation in the order of the convergence rate. This enables larger scale model learning without sacrificing convergence or accuracy.}
At each iteration of \textcolor[rgb]{0.00,0.00,0.00}{GD-SEC}, instead of directly transmitting the \textcolor[rgb]{0.00,0.00,0.00}{entire} gradient vector, each worker computes the difference between its current gradient and a linear combination of its previously transmitted gradients, and then transmits the sparsified gradient difference to the server. A key feature of \textcolor[rgb]{0.00,0.00,0.00}{GD-SEC} is that \textcolor[rgb]{0.00,0.00,0.00}{any given} component of the gradient difference vector will not be transmitted if its magnitude is not \textcolor[rgb]{0.00,0.00,0.00}{sufficiently} large. An error correction technique is used at each worker to compensate \textcolor[rgb]{0.00,0.00,0.00}{for} the error resulting from sparsification.
We prove that \textcolor[rgb]{0.00,0.00,0.00}{GD-SEC} is guaranteed to converge for strongly convex, convex, and nonconvex optimization problems \textcolor[rgb]{0.00,0.00,0.00}{with the same order of convergence rate as GD}.  Furthermore, if the objective function is strongly convex, \textcolor[rgb]{0.00,0.00,0.00}{GD-SEC} has a fast linear convergence rate. \textcolor[rgb]{0.00,0.00,0.00}{Numerical results not only validate the convergence rate of GD-SEC but also explore the communication bit savings it provides.} Given a target accuracy, \textcolor[rgb]{0.00,0.00,0.00}{GD-SEC} can significantly reduce the \textcolor[rgb]{0.00,0.00,0.00}{communications load} \textcolor[rgb]{0.00,0.00,0.00}{compared to the best existing algorithms} without slowing down the optimization process.
\end{abstract}


%
\IEEEpeerreviewmaketitle

\section{Introduction}



In recent years, there has been a surge in research and development efforts in optimization and machine
learning with the objective to optimize a performance criterion employing available data. However, traditional
centralized optimization algorithms for machine learning applications have a number of deficiencies in terms of computational resources and privacy concerns. \textcolor[rgb]{0.00,0.00,0.00}{This has motivated the study of distributed optimization algorithms}. The basic idea is to parallelize the processing load across multiple computing devices (\textcolor[rgb]{0.00,0.00,0.00}{aka} {\em workers}) where only the local processed outputs of each device are transmitted \textcolor[rgb]{0.00,0.00,0.00}{to a central server}. \textcolor[rgb]{0.00,0.00,0.00}{This paper considers a distributed optimization problem given by
\begin{align}\label{basicprobNSF2020}
\min _{\boldsymbol{\theta} \in \mathbb{R}^{d}} f(\boldsymbol{\theta}) \quad \text { with } \quad f(\boldsymbol{\theta}) :=\sum_{m \in \mathcal{M}} f_{m}(\boldsymbol{\theta}),
\end{align}
where each worker $m\in \mathcal{M}=\{1,2,...,M\}$ holds a private local function $f_{m}(\boldsymbol{\theta})$, and all workers share a globally consistent model parameter vector $\boldsymbol{\theta}\in \mathbb{R}^{d}$ which needs to be optimized. Instead of minimizing the communication overhead by slowing down the optimization process, our goal is to develop a communication-efficient distributed optimization algorithm such that all workers can obtain an optimal solution $\boldsymbol{\theta}^*$ with the same order of the convergence rate as the  classical  \textcolor[rgb]{0.00,0.00,0.00}{gradient descent (GD)} while significantly reducing the number of communication bits sent to the central server. This  is  especially  important  in  a  wireless  communication setting,  where  bandwidth  is  limited \cite{song2021federated,chen2021fedsvrg}. }

The distributed optimization problem in (\ref{basicprobNSF2020}) appears in various scenarios {that typically rely on wireless communications, so that latency, scalability, and privacy are fundamental challenges.  Applications include}
multi-agent systems \cite{bedi2019asynchronous,nedic2009distributed, bullo2009distributed,cao2012overview}, sensor networks \cite{rabbat2004distributed}, smart grids \cite{ liu2017distributed}, and distributed learning \cite{lin2017deep, konevcny2016federated, lian2017can}. 
In a distributed learning scenario, the local function $f_{m}(\boldsymbol{\theta}) $ at worker $m$ is a sum of loss functions $\ell(\boldsymbol{\theta};\mathbf{x}_n,y_n)$ for $n=1,2,...,N_m$ where $\mathbf{x}_n$ is the $n$-th feature vector and $y_n$ is the corresponding label. Various distributed optimization algorithms have been developed to solve this problem, and one {primary} method is to \textcolor[rgb]{0.00,0.00,0.00}{employ} classical GD in  a worker-server architecture setting where the central server \textcolor[rgb]{0.00,0.00,0.00}{uses} GD to update the model parameters after aggregating the gradients received from all workers, and then \textcolor[rgb]{0.00,0.00,0.00}{broadcasts}
the updated parameter back to the workers. However, overall performance is limited in practice by the fact that each worker has to transmit \textcolor[rgb]{0.00,0.00,0.00}{the entire gradient vector} at each iteration. \textcolor[rgb]{0.00,0.00,0.00}{In many problems the gradient vector may be very large, for example, on} the order of $200$ MB at each iteration \textcolor[rgb]{0.00,0.00,0.00}{when} training a deep neural network such as AlexNet \cite{alistarh2018convergence}.

Consider machine learning models where the dimension of the gradient vector is large. In wireless settings, each worker needs to \textcolor[rgb]{0.00,0.00,0.00}{load} the resulting vector into \textcolor[rgb]{0.00,0.00,0.00}{packets} for transmission. This communication step can easily become a significant bottleneck on overall performance, especially in federated learning and edge AI systems \cite{bekkerman2011scaling, shalf2010exascale, mcmahan2016communication, smith2017federated, stoica2017berkeley}. In this paper, we show empirically that the amount of transmitted bits per communication, which is approximately proportional to the number of transmitted packets, can be significantly reduced through sparisification {that} restricts each worker from transmitting less informative components of the gradient vector. The resulting algorithm, termed \textcolor[rgb]{0.00,0.00,0.00}{\textbf{G}radient \textbf{D}escent method with \textbf{S}parsification and \textbf{E}rror \textbf{C}orrection (\textbf{GD-SEC}), can reach a given objective error with much fewer communication bits compared to the classical GD while achieving the same order of convergence rate as GD}. Rigorous theoretical analysis is provided to guarantee the convergence of \textcolor[rgb]{0.00,0.00,0.00}{GD-SEC} for strongly convex, convex, and non-convex cases.

\subsection{Related Work}

\textcolor[rgb]{0.00,0.00,0.00}{There has been a surge of research for solving large-scale learning tasks like (\ref{basicprobNSF2020}), where stochastic gradient descent (SGD), a modification of the fundamental GD, is a popular method due to its low computational complexity at each iteration.  It is worth noting that the convergence rate of SGD can be inferior to that of GD, especially for a strongly convex objective function \cite{nguyen2017sarah}. For example, the convergence rate of SGD is sublinear whereas that of GD is linear under the strong convex assumption. SGD is also sensitive to the variance in the sample gradients.} Recently a class of variance reduction techniques \cite{roux2012stochastic,defazio2014saga,johnson2013accelerating,nguyen2017sarah,xiao2014proximal} have been proposed to reduce variance in the sample gradients while enjoying a faster convergence rate. In these variance reduction techniques two loops are usually needed where the outer loop computes \textcolor[rgb]{0.00,0.00,0.00}{a gradient using the entire dataset and the inner loop computes a gradient using a randomly selected sample from the entire dataset.} However, the existing work \cite{roux2012stochastic,defazio2014saga,johnson2013accelerating,nguyen2017sarah,xiao2014proximal} only reduces per-iteration complexity instead of saving communications, which is a bottleneck in distributed learning {with wireless communications}.

\textcolor[rgb]{0.00,0.00,0.00}{{In this paper we focus on a distributed GD method, called GD-SEC, where communication bit requirements  are also significantly reduced.} Different from a popular distributed algorithm called alternating direction method of multipliers (ADMM) \cite{liu2019communication,boyd2011distributed} where each worker has expensive computation cost to directly find the solution of the augmented Lagrangian function at each iteration\footnote{ The solution usually requires many GD steps to find}, the computation cost for GD-SEC is very cheap where only one GD step is needed at each round. Compared to inexact ADMM \cite{zhou2021communication,chang2014multi} where
two vectors (primal and dual) must be transmitted to the server, each worker at each iteration in GD-SEC only needs to transmit one sparsified vector, which leads to significant communication bit reduction.  }

The communication bottleneck \textcolor[rgb]{0.00,0.00,0.00}{problem} in distributed optimization has attracted significant attention in recent years. \textcolor[rgb]{0.00,0.00,0.00}{An efficient solution {should} significantly} reduce the total number of bits used in a given transmission {without impacting rapid convergence to a good estimate of $\boldsymbol{\theta}$.}
In order to provide an end-to-end speedup, a list of properties that compression methods should satisfy are proposed in \cite{agarwal2021utility}.
\textcolor[rgb]{0.00,0.00,0.00}{Quantization is a popular data compression approach which aims to approximate some quantity using a smaller number of bits to simplify processing, storage and analysis. One-bit stochastic GD (SGD) has been proposed in training deep neural networks \cite{seide20141} and multi-agent networks \cite{zhang2018distributed}. Multi-bit quantized SGD was developed where nodes can adjust the number of bits sent per iteration at the cost of possibly higher variance \cite{alistarh2017qsgd}. Variance reduction for arbitrary quantized updates has been considered to improve the convergence rate \cite{horvath2019stochastic}. A quantized decentralized GD algorithm was developed where nodes update their local decision variables by combining the quantized information received from their neighbors with their local information \cite{reisizadeh2019exact}. Quantized ADMM was proposed for decentralized machine learning \cite{elgabli2019q}. \textcolor[rgb]{0.00,0.00,0.00}{These} existing efforts on quantization \cite{seide20141, zhang2018distributed, alistarh2017qsgd, horvath2019stochastic, reisizadeh2019exact} \textcolor[rgb]{0.00,0.00,0.00}{are} very complimentary to the algorithm we consider in this paper, such that our algorithm {GD-SEC} can be combined with quantization to further improve the compression.}

{For the purpose of this paper, \textcolor[rgb]{0.00,0.00,0.00}{\textit{sparsification}}  refers to a data compression approach in machine learning where each worker sets small magnitude gradient components to zero and only transmits the remaining non-zero components. 
{This has} been shown to have significant potential in approaches that always send a fixed number of non-zero components, for example \cite{alistarh2018convergence, lin2017deep,aji2017sparse, stich2018sparsified}. Recent work in \cite{lin2017deep}  has demonstrated that the fixed number of non-zero components of each transmitted vector can be reduced by more than $600$ times through sparsification without loss of accuracy in training large-scale neural networks. Sparse communication for distributed GD was developed where $99\%$ of all entries were mapped to zero while still attaining the desired accuracy \cite{aji2017sparse} when a fixed number of components are transmitted. The general idea of setting components to zero builds on work from many different areas, including, to give a few example papers, work in statistical signal processing \cite{bryan2013making, donoho2006compressed}, matrix approximation \cite{boas2016shrinkage,stewart1993early}, distributed DNN training \cite{dutta2020discrepancy,strom2015scalable} and \textcolor[rgb]{0.00,0.00,0.00}{others}.} 
Local error correction is {also} widely used with sparsification {and significantly improves} accuracy  \cite{xu2021grace,stich2018sparsified}.

{Some recent theoretical analysis is leading to better understanding of 
when the impressive observed performance of sparsification can be guaranteed.} 
Sparsifying gradients while incorporating local error correction has been shown to provide convergence guarantees \cite{alistarh2018convergence, stich2018sparsified}.  \textcolor[rgb]{0.00,0.00,0.00}{Our GD-SEC algorithm builds on these ideas. However, in contrast, GD-SEC adaptively selects the number of the non-zero components for transmission at each iteration, differing from methods that fix the total number of non-zero components per transmission in advance \cite{alistarh2018convergence, stich2018sparsified,basu2020qsparse}. 
Furthermore, in our method each worker does not directly transmit the sparsified local gradient to the server which we found can have convergence issues\footnote{\textcolor[rgb]{0.00,0.00,0.00}{These numerical experiments are not shown} due to space limitations}.  Instead, each worker maintains a state variable to smooth its past transmitted gradients, and sparsifies the difference between its current gradient and the state variable. Numerical results given later show the advantages of our approach.}

{Reducing} the number of transmissions per iteration is also an efficient method to resolve the communication bottleneck \textcolor[rgb]{0.00,0.00,0.00}{problem}. Several researchers \cite{rago1996censoring, appadwedula2005energy, marano2006cross, patwari2003hierarchical} have studied a mode of operation called censoring, where workers will transmit only highly informative data, thus reducing worker communications. In the original work which was focused on hypothesis testing \cite{rago1996censoring, appadwedula2005energy, marano2006cross, patwari2003hierarchical}, workers transmit only very large or small likelihood ratios that contain significant information about which hypothesis is most likely to be true. For censoring-based distributed optimization, a worker will not transmit unless its local \textcolor[rgb]{0.00,0.00,0.00}{update} is sufficiently different from the previously transmitted one. In \cite{chen2018lag}, gradient descent (GD) with censoring was proposed, while  \cite{liu2019communication} proposed a communication-censored ADMM approach which restricts each worker from transmitting its local solution to neighbors if this variable is not sufficiently different from the previously transmitted one. \textcolor[rgb]{0.00,0.00,0.00}{However, in \cite{chen2018lag, liu2019communication}, workers either transmit the entire vector or they do not transmit. 
In contrast, the proposed \textcolor[rgb]{0.00,0.00,0.00}{GD-SEC} method reduces the bit-load by sparsifying the transmitted vectors and \textcolor[rgb]{0.00,0.00,0.00}{consequently GD-SEC can be} combined with censoring.}


\subsection{Our Contributions, Notation and Paper Organization}
\textcolor[rgb]{0.00,0.00,0.00}{In this paper, we develop \textcolor[rgb]{0.00,0.00,0.00}{a} {GD-SEC} method for distributed learning in the worker-server architecture. The three key ingredients of {GD-SEC} are 1) two types of state variables: one maintained at each worker to locally smooth the past transmitted gradients such that each worker’s gradient points in a direction of persistent descent; one maintained at the server to help approximate the sum of the current local gradients of workers when some components of their transmitted vectors have been suppressed, 2) a component-wise communication-censoring strategy \textcolor[rgb]{0.00,0.00,0.00}{such that a worker will not} transmit a component of its gradient vector if this component does not have sufficiently novel \textcolor[rgb]{0.00,0.00,0.00}{information, thereby reducing} the number of transmitted bits, 3) error correction compensation to mitigate the impact of the components which were not transmitted. We derive a rigorous theoretical convergence guarantee for {GD-SEC} for strongly convex, convex and nonconvex cases which shows that GD-SEC achieves the same order convergence rate as GD while saving communications.  We also numerically show the significant communication savings achieved by {GD-SEC}. We demonstrate that {GD-SEC} can achieve a linear convergence rate for a smooth and strongly convex objective function while dramatically reducing the number of transmitted bits. Numerical experimental results employing both synthetic and real datasets show that given a target objective error, {GD-SEC} is able to significantly reduce the total number of bits transmitted to the server, while maintaining a fast convergence rate compared to other distributed learning algorithms. \textcolor[rgb]{0.00,0.00,0.00}{Finally, we consider stochastic and quantized extensions of GD-SEC, and initial numerical experiments are promising}
}

Throughout this paper, we use bold lower case letters to denote column vectors. For any vector $\mathbf{z}\in \mathbb{R}^{d}$, we use $[\mathbf{z}]_i$ to denote the $i$-th component/entry of $\mathbf{z}$ with its magnitude being denoted as $|[\mathbf{z}]_i|$. We use $\|\bf{z}\|$ and $\bf{z^{\top}}$ to denote the  $\ell_2$-norm and the transpose of a column vector $\bf{z}$, respectively.

The remainder of the paper is organized as follows. In Section \ref{describeAlgorithm}, we formally describe the proposed \textcolor[rgb]{0.00,0.00,0.00}{GD-SEC} algorithm. The convergence guarantees for strongly convex, convex, and nonconvex cases are derived in Section \ref{ConvergenceGuar} with the detailed proofs developed in the appendix. In Section \ref{NumericalResult} we provide numerical experiments to illustrate the advantages of \textcolor[rgb]{0.00,0.00,0.00}{GD-SEC}. Finally, Section \ref{Conclusion} provides our conclusions.

%
%
%
%

\section{Algorithm Development}\label{describeAlgorithm}

 \textcolor[rgb]{0.00,0.00,0.00}{In this section, we describe the GD-SEC distributed algorithm so that all workers can obtain a GD solution to  the optimization problem given by (\ref{basicprobNSF2020}) with the same order of convergence rate as the classical GD while significantly reducing the number of communication bits sent to the central server.} This proposed algorithm is suitable \textcolor[rgb]{0.00,0.00,0.00}{for} large-scale distributed learning where communication can easily become \textcolor[rgb]{0.00,0.00,0.00}{a} bottleneck \textcolor[rgb]{0.00,0.00,0.00}{on} the overall performance. Before developing \textcolor[rgb]{0.00,0.00,0.00}{GD-SEC}, we first discuss the communication model and several specific strategies used in this paper.

\subsection{Communication Model and Saving Strategies}

At the start of each iteration, the server broadcasts the entire model parameter vector to all workers. Each worker will \textcolor[rgb]{0.00,0.00,0.00}{then} transmit \textcolor[rgb]{0.00,0.00,0.00}{its} sparsified vector to the server if some components have sufficiently novel information. \textcolor[rgb]{0.00,0.00,0.00}{GD-SEC} is employed to reduce the number of bits per update \textcolor[rgb]{0.00,0.00,0.00}{in} the worker-to-server uplink. In many cases the uplink (worker-to-server) may be much slower than the downlink (server-to-worker) \cite{kairouz2019advances}, and workers may have limited energy {so conserving communication resources is critical}. Unlike classical GD, in \textcolor[rgb]{0.00,0.00,0.00}{GD-SEC}, each worker does not necessarily transmit the entire vector to the server. Instead, workers transmit a sparsified vector if some components of the local \textcolor[rgb]{0.00,0.00,0.00}{update} are not sufficiently novel compared to the previous iterations, {as defined in Subsection \ref{detailedASGDdevelop}.}

We consider three specific strategies to improve the communication efficiency of classical GD.

\textcolor[rgb]{0.00,0.00,0.00}{
$\bullet$  \textbf{Adaptive sparsification}. Unlike classical sparsification methods that keep the top-$j$ components {ranked by absolute value} (with predefined $j$) while mapping the other components to zero (at each iteration), our sparsification strategy is dynamic and we do not fix the total number of components to suppress. At each iteration, components of a given worker gradient vector will \textcolor[rgb]{0.00,0.00,0.00}{be transmitted} if their values are significantly different from past values at this worker, {but otherwise these \textcolor[rgb]{0.00,0.00,0.00}{components} will be suppressed since the server can use the past transmissions to approximate the current one}. Note that the entire transmitted vector will be suppressed if none of {the} components have sufficiently novel information. {
Our adaptive sparsification is much lower complexity than the top-$j$ approaches, \textcolor[rgb]{0.00,0.00,0.00}{and} we show our performance is better than a representative top-$j$ approach in our numerical results. \textcolor[rgb]{0.00,0.00,0.00}{We remark that we carried out numerical comparisons with various choices of $j$ (not shown due to space limitations), and adapting $j$ may be a better bit-saving strategy than fixing it a priori. If true, this fundamental principle can be used to improve any approach, thus having implications far beyond our particular algorithm.}}
}

\textcolor[rgb]{0.00,0.00,0.00}{
$\bullet$ \textbf{Error correction}. Sparsification generally leads to a 
loss of information about the GD gradient updates,  inducing a bias in the estimation at the server. To reduce the bias, each worker accumulates an approximation to {its} sparsification error from the previous iteration and attempts to correct for this at each iteration. Without error correction, some components {may} be discarded without any adjustment, {whereas error correction enables an approximate adjustment at a later time}.  Intuitively, error correction enables GD-SEC to converge to the optimal solution even with aggressive sparsification. The impact of error correction on GD-SEC is further discussed in Section \ref{errorCorrImpact}.}

\textcolor[rgb]{0.00,0.00,0.00}{
$\bullet$ \textbf{State variables}. 
{Instead of directly sparsifying the local gradient $\nabla f_m({\boldsymbol{\theta}}^{k})$, which includes information the worker has already 
transmitted to the server, it is far better to 
first attempt to eliminate from consideration the information the worker has already sent to the 
server.  \textcolor[rgb]{0.00,0.00,0.00}{This is similar in spirit to classical differential plus-code modulation in communications.}  
Our state variable $\boldsymbol h_m^k$ at worker $m$ \textcolor[rgb]{0.00,0.00,0.00}{expresses} the information at worker $m$ that the server already knows. \textcolor[rgb]{0.00,0.00,0.00}{Consequently}, worker $m$ sparsifies the difference $\nabla f_m({\boldsymbol{\theta}}^{k})-\boldsymbol h_m^k$ \textcolor[rgb]{0.00,0.00,0.00}{so that the update will include new information}. 
\textcolor[rgb]{0.00,0.00,0.00}{This approach is shown to converge, and yeilds significant savings in transmitted bits in our numerical experiments in Section \ref{effectofbeta}.} We also employ a state variable at the server \textcolor[rgb]{0.00,0.00,0.00}{that aids in maintaining the overall descent direction; see also the discussion after (\ref{explainstatevariab})}. 
}}

\textcolor[rgb]{0.00,0.00,0.00}{Next we formally describe the {implementation of the GD-SEC} algorithm using these three strategies.}

\subsection{\textcolor[rgb]{0.00,0.00,0.00}{Implementation of GD-SEC}}\label{detailedASGDdevelop}

\begin{savenotes}
\begin{algorithm}
\caption{GD-SEC}\label{varianceReduceOGD}
\begin{algorithmic}[1]
\Require{step size $\alpha$, constants $\beta\in(0,1]$ and $\{\xi_i\}_{i=1}^d$ with $\xi_i>0$.}
\Ensure{${\boldsymbol{\theta}}^0,{\boldsymbol{\theta}}^1\in\mathbb{R}^{d}$, ${\boldsymbol{h}}_1^1,...,{\boldsymbol{h}}_M^1 \in\mathbb{R}^{d}$, ${\boldsymbol{e}}_1^1,...,{\boldsymbol{e}}_M^1 ={\boldsymbol{0}}$, and ${\boldsymbol{h}}^1=\sum_{m=1}^M{\boldsymbol{h}}_m^1$.}
\For{$k=1, 2,..., K$}
\State Server broadcasts $\boldsymbol{\theta}^k$ to all workers at the start of iteration $k$.
\For{$m=1, 2,..., M$}
\State Worker $m$ computes $\nabla f_m({\boldsymbol{\theta}}^k)$, sets $\Delta_m^k=\nabla f_m({\boldsymbol{\theta}}^k)-{\boldsymbol{h}}_m^k+{\boldsymbol{e}}_m^k$ and sets $\hat{\Delta}_m^k=\Delta_m^k$.
\For{$i=1, 2,..., d$}
\If {$|[\Delta_m^k]_{i}|\le \frac{\xi_i}{M} |[{\boldsymbol{\theta}}^k-{\boldsymbol{\theta}}^{k-1}]_{i}|$}
\State Worker $m$ sets $[\hat{\Delta}_m^k]_{i}=0$.
\EndIf
\EndFor
\If {$\|\hat{\Delta}_m^k\|\ne0$}
\State Worker $m$ transmits $\hat{\Delta}_m^k$ back to the server.
\Else
\State Worker $m$ transmits \textcolor[rgb]{0.00,0.00,0.00}{nothing}.
\EndIf
\State Worker $m$ updates ${\boldsymbol{h}}_m^{k+1}={\boldsymbol{h}}_m^{k}+\beta \hat{\Delta}_m^k$ and ${\boldsymbol{e}}_m^{k+1}= \Delta_m^k-\hat{\Delta}_m^k$.
\EndFor
\State Server computes $\hat\Delta^k =  \sum_{m=1}^M\hat{\Delta}_m^k$ with $\hat{\Delta}_m^k=\boldsymbol{0}$ if \textcolor[rgb]{0.00,0.00,0.00}{nothing} from worker $m$ is received.
\State Server updates ${\boldsymbol{\theta}}^{k+1}={\boldsymbol{\theta}}^{k} -\alpha({\boldsymbol{h}}^k + \hat\Delta^k )$.
\State Server updates ${{\boldsymbol h}^{k+1}} = {{\boldsymbol h}^{k}} + \beta\hat\Delta^k$.
\EndFor
\end{algorithmic}
\end{algorithm}
\end{savenotes}

\textcolor[rgb]{0.00,0.00,0.00}{GD-SEC} is summarized as Algorithm \ref{varianceReduceOGD}.
  At iteration $k$, worker $m$ ($m=1,2,..,M$) maintains 3 vectors. The first is a parameter vector $\boldsymbol{\theta}^k$ which is received from the server at the start of iteration $k$, the second is a state variable ${\boldsymbol{h}}_m^k$ that accumulates all the information transmitted up to iteration $(k-1)$, \textcolor[rgb]{0.00,0.00,0.00}{and the third is the error vector ${\boldsymbol{e}}_m^{k}= \Delta_m^{k-1}-\hat{\Delta}_m^{k-1}=\nabla f_m({\boldsymbol{\theta}}^{k-1})-{\boldsymbol{h}}_m^{k-1}+{\boldsymbol{e}}_m^{k-1}$ which accumulates the error from the previous iteration.} An important feature of \textcolor[rgb]{0.00,0.00,0.00}{GD-SEC} is that worker $m$ will not transmit the $i$-th component of the difference $\Delta_m^k$
if the following {\bf{\textcolor[rgb]{0.00,0.00,0.00}{GD-SEC  transmission stopping condition}}} is satisfied
\begin{align}\label{ASGD-stopping condition} |[\Delta_m^k]_{i}|\le \frac{\xi_i}{M} |[{\boldsymbol{\theta}}^k-{\boldsymbol{\theta}}^{k-1}]_{i}|,
\end{align}
where $\xi_i$ is a constant {for $i=1,2,...,d$}, and
$\Delta_m^k:=\nabla f_m({\boldsymbol{\theta}}^k)-{\boldsymbol{h}}_m^k+{\boldsymbol{e}}_m^k$. \textcolor[rgb]{0.00,0.00,0.00}{Different from a hard-threshold sparsifier which is usually suboptimal, the threshold in (\ref{ASGD-stopping condition}) is adaptively changing at each iteration}. 
Note that worker $m$ will suppress the entire transmitted vector if all components of $\Delta_m^k$ satisfy (\ref{ASGD-stopping condition}). Otherwise, worker $m$ transmits the sparsified vector $\hat{\Delta}_m^k$ to the server where \textcolor[rgb]{0.00,0.00,0.00}{for all $i=1,2,...,d$
\begin{align}\label{eq:alpha}
[\hat{\Delta}_m^k]_i :=\left\{
           \begin{array}{lcl}
    \quad 0,\quad   &\mbox{if (\ref{ASGD-stopping condition}) is true}  \\
    {[{\Delta}_m^k]_i},\quad  &\mbox{\textcolor[rgb]{0.00,0.00,0.00}{otherwise}}.
           \end{array}
        \right.
\end{align}}Then worker $m$ locally updates its state variable ${\boldsymbol{h}}_m^{k+1}={\boldsymbol{h}}_m^{k}+\beta \hat{\Delta}_m^k$ with $\beta\in(0,1]$ and records its error vector ${\boldsymbol{e}}_m^{k+1}= \Delta_m^k-\hat{\Delta}_m^k$.  Note that the error ${\boldsymbol{e}}_m^{k+1}$ will be added to $\nabla f_m({\boldsymbol{\theta}}^{k+1})$ before sparsification at iteration $(k+1)$. 

In order to intuitively \textcolor[rgb]{0.00,0.00,0.00}{understand} the impact of $\beta$ on ${\boldsymbol{h}}_m^{k+1}$ with $\beta\in(0,1]$, suppose all components are transmitted during the first $k$ iterations. \textcolor[rgb]{0.00,0.00,0.00}{Then,} we obtain $\hat{\Delta}_m^k={\Delta}_m^k=\nabla f_m({\boldsymbol{\theta}}^k)-{\boldsymbol{h}}_m^k$
and can rewrite ${\boldsymbol{h}}_m^{k+1}$ as
\begin{align}
{\boldsymbol{h}}_m^{k+1}&={\boldsymbol{h}}_m^{k}+\beta(\nabla f_m({\boldsymbol{\theta}}^k)-{\boldsymbol{h}}_m^k)\label{iteratRela}\\
&=(1-\beta)^k{\boldsymbol{h}}_m^{1} + \sum_{j=1}^{k}(1-\beta)^{k-j}\beta\nabla f_m({\boldsymbol{\theta}}^j),\label{explainstatevariab}
\end{align}
\textcolor[rgb]{0.00,0.00,0.00}{where (\ref{explainstatevariab}) is obtained by recursively applying (\ref{iteratRela}) on ${\boldsymbol{h}}_m^{k}$, ${\boldsymbol{h}}_m^{k-1}$,..., ${\boldsymbol{h}}_m^{1}$. The result in (\ref{explainstatevariab}) indicates that the state variable ${\boldsymbol{h}}_m^{k+1}$ at worker $m$ is a linear combination of the first state variable ${\boldsymbol{h}}_m^{1}$ and all of its past gradients $\{\nabla f_m({\boldsymbol{\theta}}^j)\}_{j=1}^{k-1}$ which implies that ${\boldsymbol{h}}_m^{k+1}$ can accumulate contributions in directions of persistent (long term) descent of the local function $f_m({\boldsymbol{\theta}})$. To be more precise, if $\beta\ne 1$, 
this makes each worker move in a direction which decreases the function value from a long-term prospective. From (4) ${\boldsymbol{h}}_m^{k}$ reduces to the last transmitted gradient $\nabla f_m({\boldsymbol{\theta}}^{k-1})$ if we set $\beta=1$. Some numerical results in Subsection \ref{effectofbeta} show that accumulating all the  past  gradients (setting $\beta\ne 1$) \textcolor[rgb]{0.00,0.00,0.00}{leads} to better communication savings. Small oscillations are eliminated, and the larger time scale descent direction can be expressed concisely, saving communications bits.   This state variable update strategy in (\ref{explainstatevariab}) seems similar to the momentum term in the heavy ball method, \textcolor[rgb]{0.00,0.00,0.00}{but note that} with sparsification the state variable ${\boldsymbol{h}}_m^{k+1}$ actually accumulates the transmitted sparsified vector $\{\hat{\Delta}_m^j\}_{j=1}^k$. \textcolor[rgb]{0.00,0.00,0.00}{Note also that we employ a state variable at each worker and the server, unlike the classical use of a momentum term only at the server.} }

After receiving $\hat{\Delta}_m^k$ from all workers, the server updates ${\boldsymbol{\theta}}^k$ via the following {\bf {\textcolor[rgb]{0.00,0.00,0.00}{GD-SEC} update rule}}
\begin{align}\label{updateruleSpar}
{\boldsymbol{\theta}}^{k+1}&={\boldsymbol{\theta}}^{k} -\alpha({\boldsymbol{h}}^k + \hat\Delta^k )\notag\\
&\text{ with } {\boldsymbol{h}}^k= \sum_{m=1}^M {\boldsymbol{h}}_m^k \text{ and } \hat\Delta^k=\sum_{m=1}^M \hat\Delta_m^k,
\end{align}
where $\alpha$ is the step size and $({\boldsymbol{h}}_m^k + \hat\Delta_m^k )$ can be regarded as the approximated gradient for worker $m$ at iteration $k$. \textcolor[rgb]{0.00,0.00,0.00}{
{If} all components are transmitted during the first $k$ iterations, (\ref{updateruleSpar}) reduces to the classical gradient descent update rule.   
Finally, the server updates its state variable ${{\boldsymbol h}^{k+1}} = {{\boldsymbol h}^{k}} + \beta\hat\Delta^k$ based on the sparsified $\hat\Delta^k$. 
As will be noted from Algorithm 1, 
the server can compute ${\boldsymbol{h}}^k$ without requiring each worker $m$ to  transmit ${\boldsymbol{h}}_m^k$ in (\ref{updateruleSpar}) since the server and every worker {will update their own state variables} in a similar manner. The implementation of GD-SEC requires each worker be synchronized at each iteration {that, for example,} can be accomplished by using the synchronous federated learning protocols \cite{bonawitz2019towards,sergeev2018horovod}. }



\section{Convergence Analysis}\label{ConvergenceGuar}

In this section, the convergence rate for Algorithm \ref{varianceReduceOGD} is developed. With the proper choice of the constants $\{\xi_i\}_{i=1}^d$, $\beta$ and $\alpha$, the  \textcolor[rgb]{0.00,0.00,0.00}{GD-SEC} method is shown to converge for strongly convex, convex, and nonconvex objective functions while reducing, often significantly, the total number of bits per communication to reach a given objective error. Furthermore, if the local functions are strongly convex or convex, \textcolor[rgb]{0.00,0.00,0.00}{GD-SEC} has a linear or sublinear convergence rate, respectively. \textcolor[rgb]{0.00,0.00,0.00}{Let $\boldsymbol{u_i}=[0,...,0,1,0,...,0]$ with only the $i$-th component being one.}
 \textcolor[rgb]{0.00,0.00,0.00}{The following assumptions enable a rigorous analysis.}
\begin{assumption}\label{OHBassumption1}
{In \emph{(\ref{basicprobNSF2020})}, $f(\boldsymbol{\theta})$ is coordinate-wise $L^i$-smooth and coercive. \textcolor[rgb]{0.00,0.00,0.00}{Coordinate-wise $L^i$-smooth
implies there exists a constant $L^i>0$ such that $|[\nabla f(\boldsymbol{\theta}+t {\boldsymbol{u_i}})- \nabla f(\boldsymbol{\theta})]_i|\le L^i|t|$  for all possible values of {$t$} for all $i=1,2,...,d$.  Coercive implies  $\lim_{\|\boldsymbol{\theta}\|\rightarrow\infty}f(\boldsymbol{\theta})=+\infty$ \emph{\cite{peressini1988mathematics}}.} Similarly, local function $f_m(\boldsymbol{\theta})$ at worker $m$ is coordinate-wise $L_m^i$ smooth. \textcolor[rgb]{0.00,0.00,0.00}{ 
Note that coordinate-wise smoothness of $f(\boldsymbol{\theta})$ and $f_m(\boldsymbol{\theta})$ implies $f(\boldsymbol{\theta})$ and $f_m(\boldsymbol{\theta})$ are smooth with smoothness constants being denoted as $L$ and $L_m$, respectively. To be exact, $f(\boldsymbol{\theta})$ is $L$ smooth implies there exists a constant $L$ such that $\|\nabla f(\boldsymbol{\theta}_1)-\nabla f(\boldsymbol{\theta}_2)\|\le L\|\boldsymbol{\theta}_1-\boldsymbol{\theta}_2\|,\ \forall \ \boldsymbol{\theta}_1, \boldsymbol{\theta}_2$. Similar definition applies to $L_m$.}}
\end{assumption}
\begin{assumption}\label{OHBassumption2}
In \emph{(\ref{basicprobNSF2020})}, $f(\boldsymbol{\theta})$ is $\mu$-strongly convex. This implies there exists a constant $\mu>0$ such that
$f(\boldsymbol{\theta}_1)\ge f(\boldsymbol{\theta}_2) + \nabla f(\boldsymbol{\theta}_2) ^{\top}(\boldsymbol{\theta}_1-\boldsymbol{\theta}_2)  + \frac{\mu}{2}\|\boldsymbol{\theta}_1-\boldsymbol{\theta}_2\|^2, \forall \ \boldsymbol{\theta}_1, \boldsymbol{\theta}_2$ \emph{\cite{nesterov2013introductory}}.
\end{assumption}
\begin{assumption}\label{OHBassumption3}
In \emph{(\ref{basicprobNSF2020})}, $f(\boldsymbol{\theta})$ is convex. This implies $f(\boldsymbol{\theta})$ satisfies $f(\lambda\boldsymbol{\theta}_1 + (1-\lambda)\boldsymbol{\theta}_2)\le \lambda f(\boldsymbol{\theta}_1) + (1-\lambda)f(\boldsymbol{\theta}_2), \forall \ \boldsymbol{\theta}_1, \boldsymbol{\theta}_2$ and $0\le\lambda\le 1$ \emph{\cite{peressini1988mathematics}}.
\end{assumption}

To prove \textcolor[rgb]{0.00,0.00,0.00}{GD-SEC} converges, we employ a Lyapunov function\footnote{\textcolor[rgb]{0.00,0.00,0.00}
The function in (\ref{Laypufunc}) is a Lyapunov function \cite{taylor2018lyapunov} as it is decreasing (shown in Lemma \ref{descentLyapunovfunv}), non-negative for all $\boldsymbol{\theta}^{k}$, equals zero if and only if $\boldsymbol{\theta}^{k}=\boldsymbol{\theta}^{*}$, 
and approaches infinity as $\boldsymbol{\theta}^{k}\rightarrow\infty$. \textcolor[rgb]{0.00,0.00,0.00}{{It is not necessary for} the loss function $f\left(\boldsymbol{\theta}\right)$ to be a Lyapunov function.}}  defined as
\begin{align}\label{Laypufunc}
\mathbb{L}^{k}:&=f\left(\boldsymbol{\theta}^{k}\right)-f\left(\boldsymbol{\theta}^{*}\right)+ {\beta}_1 \left\|\boldsymbol{\theta}^{k}-\boldsymbol{\theta}^{k-1}\right\|^{2}\notag\\
&\quad+ {\beta}_2 \left\|\boldsymbol{\theta}^{k-1}-\boldsymbol{\theta}^{k-2}\right\|^{2},
\end{align}
where $\boldsymbol{\theta}^{*}$ is the optimal solution to the optimization problem in (\ref{basicprobNSF2020}) and $\beta_1$ and $\beta_2$ are non-negative constants that will be specified later. If we set $\beta=\beta_1=\beta_2=0$, $\boldsymbol{h}_m^1=\boldsymbol{0}$ for all $m=1,2,...,M$, and $\xi_i\le0$ for all $i=1,2,...,d$, then \textcolor[rgb]{0.00,0.00,0.00}{GD-SEC} reduces to the classical GD method and $\mathbb{L}^k$ can be used to analyze the GD method. However, when some communications are suppressed and the transmitted vector is sparsified, then the following lemma describes the behavior of the Lyapunov function defined in (\ref{Laypufunc}).

\begin{lemma}\label{descentLyapunovfunv}
Under Assumption \ref{OHBassumption1}, if the constants $\alpha$, $\xi=\max_i\xi_i$, $\beta_1$ and $\beta_2$ are chosen so that
\begin{align}
\gamma&\ge 0\ \mbox{with}\ \gamma:=\frac{L}{2}-\frac{1}{2 \alpha}+\beta_1, \label{etacondition5}\\
\sigma_0&\ge 0 \ \mbox{with}\ \sigma_0:= \frac{\alpha}{2} - \gamma(1+\rho)\alpha^2,\label{DDScondition1}\\
\sigma_1&\ge 0 \ \mbox{with}\notag\\
\sigma_1&:= -\beta_2+\beta_1
 - (1+\rho_2)\xi^2\left( \frac{\alpha}{2} + \gamma(1+\rho^{-1})\alpha^2 \right),\label{condition2}\\
\sigma_2&\ge 0 \ \mbox{with}\notag\\
\sigma_2&:= \beta_2 - (1+\rho_2^{-1})\xi^2\left( \frac{\alpha}{2} + \gamma(1+\rho^{-1})\alpha^2 \right),\label{condition3}
\end{align}
\textcolor[rgb]{0.00,0.00,0.00}{where the constants $\rho>0$ and $\rho_2>0$,}
 then the Lyapunov function follows
\begin{align}\label{nonincreasing}
\mathbb{L}^{k+1}-\mathbb{L}^k &\leq-\sigma_{0}\left\|\nabla f\left(\boldsymbol{\theta}^{k}\right)\right\|^{2}
-\sigma_{1}\left\|\boldsymbol{\theta}^{k}-\boldsymbol{\theta}^{k-1}\right\|^{2}\notag\\
&-\sigma_{2}\left\|\boldsymbol{\theta}^{k-1}-\boldsymbol{\theta}^{k-2}\right\|^{2},
\end{align}
where constants $\sigma_{0}\ge0$, $\sigma_{1}\ge0$ and $\sigma_{2}\ge0$ depend on $\alpha$, $\xi=\max_i\xi_i$, $\beta_1$ and $\beta_2$. The result in (\ref{nonincreasing}) indicates $\mathbb{L}^{k+1}\le\mathbb{L}^{k}$.
\end{lemma}
The detailed proof of Lemma \ref{descentLyapunovfunv} is provided in Appendix \ref{decreaseproperty}. If we set $\big(\beta_1-\frac{1-\alpha L}{2\alpha}\big)=0$, then (\ref{etacondition5})--(\ref{condition3}) are equivalent to
\begin{align}\label{addconstraintLayp}
\alpha&\le\frac{1}{L}, \mbox{and}\ \xi\le\min\Big\{\sqrt{\frac{2(\beta_1-\beta_2)}{(1+\rho_2)\alpha}}, \sqrt{\frac{2\beta_2}{(1+\rho_2^{-1})\alpha}}\Big\}.
\end{align}
Note that we can use (\ref{addconstraintLayp}) to find an uncountable number of parameters $\alpha$ and $\xi$ which would guarantee (\ref{etacondition5})--(\ref{condition3}) and some examples of parameter choices can be found in Appendix \ref{decreaseproperty}. In the following, we will show that the Lyapunov function in (\ref{Laypufunc}) has a linear convergence rate under certain conditions.

\begin{theorem}(\textbf{strongly convex})\label{stonglyconvextheo}
Under \emph{Assumptions \ref{OHBassumption1}} and \emph{\ref{OHBassumption2}}, if constants $\alpha$ and $\xi$ are properly selected such that (\ref{etacondition5})--(\ref{condition3}) are satisfied with $\sigma_0\sigma_1\sigma_2\ne0$, then there exists a constant $c(\alpha,\xi)\in(0,1)$ such that at iteration $k$,
\begin{align}\label{linearconveLay1243}
 \mathbb{L}^{k+1}\le \big(1-c(\alpha,\xi)\big)\mathbb{L}^{k},
\end{align}
where $c(\alpha,\xi)=\min\big\{ 2\sigma_0\mu, \sigma_1/\beta_1, \sigma_2/\beta_2\big\}$ with $\sigma_0$, $\sigma_1$ and $\sigma_2$ defined in (\ref{DDScondition1})--(\ref{condition3}), respectively.
The result in (\ref{linearconveLay1243}) implies
\begin{align}\label{linerratefunctiondecrease}
f(\boldsymbol{\theta}^k) - f\left(\boldsymbol{\theta}^*\right)\le \big(1-c(\alpha,\xi)\big)^{k}\mathbb{L}^0.
\end{align}
\end{theorem}
The detailed proof of Theorem \ref{stonglyconvextheo} is provided in Appendix \ref{linearstronglyconvex}.
Theorem \ref{stonglyconvextheo} implies that the \textcolor[rgb]{0.00,0.00,0.00}{GD-SEC} algorithm can attain a linear convergence rate which is on the same order as the classical GD method \cite{nesterov2013introductory} under {Assumptions \ref{OHBassumption1}} and {\ref{OHBassumption2}}. As a particular example of $c(\alpha,\xi)$, if we choose $\rho_2=1$,
$\delta\in(0,1), \beta\in(0,1],\alpha=\frac{1-\delta}{L}, \xi^2\le\frac{1-\alpha\mu}{\alpha},   \beta_2=\frac{\alpha\xi^2}{1-\alpha\mu}, \mbox{and}\ \beta_1=\beta_2+\frac{1}{1-\alpha\mu}$,
then
\begin{align}\label{CHBcomxplx}
c(\alpha,\xi) = \frac{1-\delta}{L/\mu}
\end{align}
whose value is exactly the same as \textcolor[rgb]{0.00,0.00,0.00}{that for} GD if we also choose $\alpha={(1-\delta)}/{L}$ in GD. \textcolor[rgb]{0.00,0.00,0.00}{Given this specific setting, the iteration complexity being defined as the number of iterations to achieve a user defined error $\epsilon\ge{\mathbb{L}^{k+1}}/{\mathbb{L}^{1}}$ is $\mathbb{I}_{GD-SEC}(\epsilon)= 1/(\alpha\mu)\log(\epsilon^{-1})$ {as shown in} in  Appendix \ref{linearstronglyconvex}. } Additionally, there are an uncountable number of other parameter settings in \textcolor[rgb]{0.00,0.00,0.00}{GD-SEC} that provide the same order convergence rate as GD, but we omit further discussion \textcolor[rgb]{0.00,0.00,0.00}{for brevity}. In addition to the strongly convex case, we also provide a convergence guarantee for \textcolor[rgb]{0.00,0.00,0.00}{GD-SEC} for general convex and nonconvex objective functions with the detailed proofs given in Appendix \ref{convergeconv} and Appendix \ref{convergenonconv}.

\begin{theorem}(\textbf{convex})\label{convextheo}
Under \emph{Assumptions \ref{OHBassumption1} and \ref{OHBassumption3}}, if the constants $\alpha$ and $\xi$ are chosen so that (\ref{etacondition5})--(\ref{condition3}) are satisfied with $\sigma_0\sigma_1\sigma_2\ne0$, then
\begin{align}
f(\boldsymbol{\theta}^{k})-f\left(\boldsymbol{\theta}^{*}\right)=\mathcal{O}(1 / k).
\end{align}
\end{theorem}

\begin{theorem}(\textbf{nonconvex})\label{convergenceTheorem3}
Under \emph{Assumption \ref{OHBassumption1}}, if the constants $\alpha$ and $\xi$ are chosen so that (\ref{etacondition5})--(\ref{condition3}) are satisfied with $\sigma_0\sigma_1\sigma_2\ne0$, then
\textcolor[rgb]{0.00,0.00,0.00}{\begin{align}\label{nonconvexfunc13}
\min_{1\le k'\le k}\left\|\nabla f(\boldsymbol{\theta}^{k'})\right\|^{2}=\mathcal{O}(1 / k).
\end{align}}
\end{theorem}

\textcolor[rgb]{0.00,0.00,0.00}{Theorem \ref{convextheo} and Theorem \ref{convergenceTheorem3} indicate GD-SEC can achieve the same order of convergence rate as GD for convex and nonconvex objective functions.} \textcolor[rgb]{0.00,0.00,0.00}{Note that these convergence results are obtained using a fixed step size, whereas methods that fix the number of components for transmission are not guaranteed to converge using a fixed step size \cite{stich2018sparsified}.} 


\section{Numerical Results}\label{NumericalResult}

To validate the effectiveness in reducing \textcolor[rgb]{0.00,0.00,0.00}{the total number of transmitted bits up to the current iteration} and the theoretical results on convergence analysis, the empirical performance of \textcolor[rgb]{0.00,0.00,0.00}{GD-SEC} is evaluated through numerical experiments \textcolor[rgb]{0.00,0.00,0.00}{where the total number of transmitted bits are defined as a communication cost}. We employ $32$ bits to represent the value of an entry with negligible loss in precision, and apply the Run-Length Encoding (RLE) algorithm \textcolor[rgb]{0.00,0.00,0.00}{\cite{brooks2020run}} to encode the indices of the non-zero components of $\hat\Delta_m^k$ in \textcolor[rgb]{0.00,0.00,0.00}{GD-SEC}. In particular, by counting the number of consecutive zeros between two non-zero components, we can \textcolor[rgb]{0.00,0.00,0.00}{use RLE to efficiently encode} the locations of all non-zero components compared to the naive encoding (index and value pairs). To benchmark \textcolor[rgb]{0.00,0.00,0.00}{GD-SEC} with RLE, we 
{compare with} the following four algorithms:

$\bullet$ GD (baseline) \cite{nesterov2013introductory}: At each iteration $k$, each worker $m$ transmits the original gradient vector $\nabla f_m(\boldsymbol{\theta}^{k})\in\mathbb{R}^{d}$ to the server. Then the server aggregates all gradients $\{\nabla f_m(\boldsymbol{\theta}^{k})\}_{m=1}^M$ and runs gradient descent to update $\boldsymbol{\theta}^{k}$. \textcolor[rgb]{0.00,0.00,0.00}{This requires $32\times d$ bits to represent the gradient vector for transmission.}

$\bullet$ top-$j$ with RLE \cite{stich2018sparsified}: At each iteration $k$, each worker $m$ only transmits 
\textcolor[rgb]{0.00,0.00,0.00}{
the top $j$ entries ($j$ is fixed), in terms of absolute value of the gradient vector components} and then  employs error correction to accumulate errors from the entries which are not transmitted. \textcolor[rgb]{0.00,0.00,0.00}{Note that top-$j$ with error correction is only guaranteed to converge for strongly convex functions with gradually decreasing step sizes \cite{stich2018sparsified}.
} RLE is applied to encode the indices of the non-zero entries of the transmitted sparsified vector.

$\bullet$ \textcolor[rgb]{0.00,0.00,0.00}{Censoring-based GD (CGD) with RLE:  At each iteration $k$ of the original CGD \cite{chen2018lag},   each worker $m$ transmits the entire gradient vector $\nabla f_m(\boldsymbol{\theta}^{k})$ to the server if the current gradient $\nabla f_m(\boldsymbol{\theta}^{k})$ is sufficiently different from its previously transmitted one; otherwise, worker $m$ discards its current gradient $\nabla f_m(\boldsymbol{\theta}^{k})$ and transmits nothing, unlike \textcolor[rgb]{0.00,0.00,0.00}{GD-SEC}
where worker $m$ locally accumulates the current gradient $\nabla f_m(\boldsymbol{\theta}^{k})$. A sufficient difference is judged 
by the size of $\tilde{\xi}\|\boldsymbol{\theta}^{k}-\boldsymbol{\theta}^{k-1}\|/M$ with $\tilde{\xi}>0$.  This approach requires $32\times d$ bits to represent the gradient vector for transmission. Here we employ Censoring-based GD (CGD) with RLE which is a variant of CGD where we use RLE to encode the indices of the non-zero entries of the transmitted vector.}

$\bullet$ \textcolor[rgb]{0.00,0.00,0.00}{Quantized GD (QGD) \cite{alistarh2017qsgd,reisizadeh2020fedpaq}: Each worker $m$ transmits the quantized version of the gradient vector $\nabla f_m(\boldsymbol{\theta}^{k})\in\mathbb{R}^{d}$ to the server. For any $\boldsymbol{v}\in\mathbb{R}^{d}$, the low-precision unbiased quantizer output is defined as $Q_s([\boldsymbol{v}]_i)=\|\boldsymbol{v}\|\cdot\text{sign}([\boldsymbol{v}]_i)\cdot\eta_i(\boldsymbol{v},s)$, where $s$ is the total number of quantization intervals (bins) and  $\{ \eta_i(\boldsymbol{v},s), \forall i \}$ are independent random variables each taking on value $(l+1)/s$ with probability $p=|[\boldsymbol{v}]_i|\cdot s/\|\boldsymbol{v}\|-l$ and $l/s$ with probability $(1-p)$. Here $0\le l<s$ is an integer describing which quantization interval the current observation falls into such that $|[\boldsymbol{v}]_i|/\|\boldsymbol{v}\|\in[l/s,(l+1)/s]$. In the following results, for each transmitted vector, we employ $8$ bits and $1$ bit to represent the value and the sign of each non-zero component, respectively. \textcolor[rgb]{0.00,0.00,0.00}{An extra 32} bits are used to approximate $\|\boldsymbol{v}\|$ when $\|\boldsymbol{v}\|\ne0$.  }

$\bullet$ \textcolor[rgb]{0.00,0.00,0.00}{Nonuniform sampling of incremental aggregated gradient (NoUnif-IAG) \cite{schmidt2017minimizing}: At each iteration of NoUnif-IAG, only one worker is selected to transmit its current fresh gradient to the server, and the server aggregates this new local gradient with the past transmitted gradients from other workers to obtain the final entire gradient. The probability of selecting worker $m$ to transmit is set to be $L_m/\sum_{m=1}^ML_m$. }

We consider four optimization problems: regularized linear regression (convex), regularized logistic regression (strongly convex), lasso regression (nondifferentiable) and non-linear least squares \cite{xu2020second} (nonconvex). The objective error $f(\boldsymbol{\theta}^k)-f(\boldsymbol{\theta}^*)$ is used to evaluate the algorithm progress for all tasks. \textcolor[rgb]{0.00,0.00,0.00}{We use one server and $M=5$ workers.} Except as stated elsewhere, in \textcolor[rgb]{0.00,0.00,0.00}{GD-SEC} we restrict the threshold $\xi_1=\xi_2=...=\xi_d=\xi$ in (\ref{ASGD-stopping condition}), set the constant $\beta=0.01$ in (\ref{explainstatevariab}) to average the previously transmitted gradients, and initialize the state variable $\boldsymbol{h}_m^1=\boldsymbol{0}$ for each worker $m$. The constant step size is tuned for GD, and \textcolor[rgb]{0.00,0.00,0.00}{all other algorithms use the same step size except top-$j$ and NoUnif-IAG\footnote{Based on experiments, top-$j$ does not converge using this constant step and NoUnif-IAG is not stable when using the same step size as GD.} to enable fair comparison in a given learning task. According to top-$j$ in \cite{stich2018sparsified}, we {use} a decreasing step size $\alpha_k=\gamma_0(1+\gamma_0\lambda k)^{-1}$ with constants $\gamma_0$ and $j$ being tuned}. For \textcolor[rgb]{0.00,0.00,0.00}{GD-SEC} and CGD, we tune the best thresholds $\tilde{\xi}$ and $\xi$ (\textcolor[rgb]{0.00,0.00,0.00}{restricted} to be integer) such that each of them can not only converge quickly (similar to  GD) but also save the largest total amount of transmitted bits.

\subsection{Algorithm Comparison in Regularized Linear Regression}\label{linearReMNIST}

\begin{figure}[!t]
\centering
\includegraphics[width=3.6in]{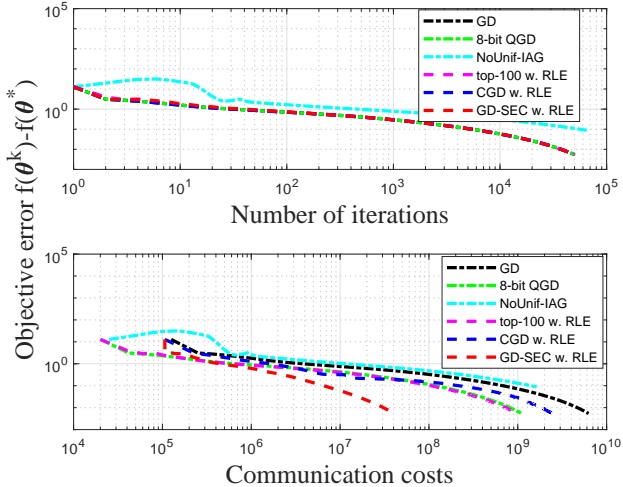}
\caption{\textcolor[rgb]{0.00,0.00,0.00}{Objective error versus the number of iterations and the communication costs for regularized linear regression in the \emph{MNIST} dataset (using 2000 data samples).}}
\label{conv_fig1}
\end{figure}

We first consider the convergence and \textcolor[rgb]{0.00,0.00,0.00}{total number of transmitted bits} of \textcolor[rgb]{0.00,0.00,0.00}{GD-SEC} for regularized linear regression \textcolor[rgb]{0.00,0.00,0.00}{using a} subset of the  MNIST dataset \cite{lecun1998gradient}. Specifically, we select the first $2000$ data samples and the corresponding labels as our training data, and evenly split them \textcolor[rgb]{0.00,0.00,0.00}{among} five workers ($M=5$). In the regularized linear regression problem, the local function defined in (\ref{basicprobNSF2020}) for worker $m$ for $m=1,2,...,M$ is
\begin{align}\label{linearloss}
f_{m}(\boldsymbol{\theta}) :=\frac{1}{2N}\sum_{n=1}^{{N}_{m}}\left(y_{n}-\mathbf{x}_{n}^{\top} \boldsymbol{\theta}\right)^{2} + \frac{\lambda}{2M}\|\boldsymbol{\theta}\|^2,
\end{align}
where $N=2000$ is the total number of samples, ${N}_{m}=400$ is the number of samples at worker $m$, $y_{n}$ is the $n$-th label, $\mathbf{x}_{n}$ is the $n$-th feature vector, $\boldsymbol{\theta}$ is the parameter vector, and $\lambda$ is the regularization parameter which is set to be $\textcolor[rgb]{0.00,0.00,0.00}{\lambda=}1/N=1/2000$ following \cite{schmidt2017minimizing}. \textcolor[rgb]{0.00,0.00,0.00}{In the experiments, the step size $\alpha=1/L\approx0.0258$ is tuned for GD, and all other algorithms use the same step size except top-$j$ and NoUnif-IAG. In top-$j$, we have tried all combinations for $j=\{1,10,100\}$ and $\gamma_0=\{0.01,0.1,1,10\}$ and found that setting $j=100$ and $\gamma_0=0.01$ achieves the best convergence performance. In order to have stability, we set $\alpha=1/(2ML)$ for NoUnif-IAG}. We set $\xi/M=800$ in \textcolor[rgb]{0.00,0.00,0.00}{GD-SEC}, and $\tilde\xi/M=1$ in CGD. The performance is shown in Fig. \ref{conv_fig1}. The results in Fig. \ref{conv_fig1} indicate that \textcolor[rgb]{0.00,0.00,0.00}{GD-SEC} outperforms \textcolor[rgb]{0.00,0.00,0.00}{the} other algorithms in \textcolor[rgb]{0.00,0.00,0.00}{reducing the total number of transmitted bits} while achieving nearly the same convergence as the classical GD. \textcolor[rgb]{0.00,0.00,0.00}{Given an objective error of $5.4\times10^{-3}$,  \textcolor[rgb]{0.00,0.00,0.00}{GD-SEC} saves about $99.34\%$ of the total transmitted bits compared to classical GD}.


\subsection{Algorithm Comparison in Regularized Logistic Regression}

\begin{figure}[!t]
\centering
\includegraphics[width=3.6in]{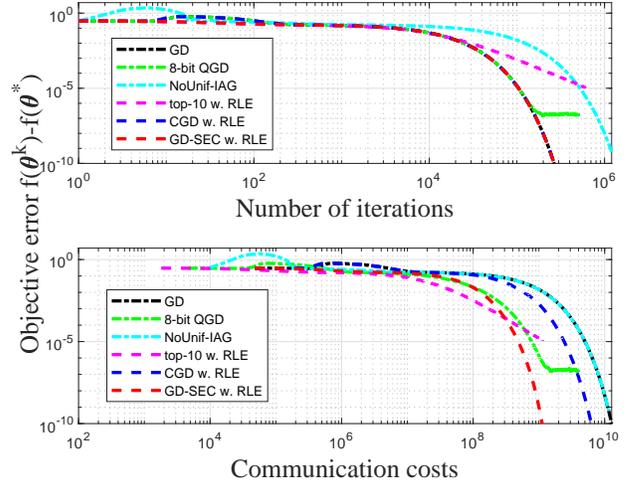}
\caption{\textcolor[rgb]{0.00,0.00,0.00}{Objective error versus the number of iterations and  the communication costs for regularized logistic regression in the synthetic dataset}.}
\label{conv_fig2}
\end{figure}

Now we consider the performance of \textcolor[rgb]{0.00,0.00,0.00}{GD-SEC} 
\textcolor[rgb]{0.00,0.00,0.00}{
for synthetic datasets and regularized} logistic regression. \textcolor[rgb]{0.00,0.00,0.00}{Here}, the local function defined in (\ref{basicprobNSF2020}) for worker $m$, \textcolor[rgb]{0.00,0.00,0.00}{$m=1,2,...,M$}, is
\begin{align}\label{logisticloss}
f_{m}(\boldsymbol{\theta}) :=\frac{1}{N}\sum_{n=1}^{{N}_{m}} \log \left(1+\exp \left(-y_{n} \mathbf{x}_{n}^{\top} \boldsymbol{\theta}\right)\right)+\frac{\lambda}{2M}\|\boldsymbol{\theta}\|^{2},
\end{align}
where $\lambda=1/N$ and $M=5$. For each worker $m$, we randomly generate an independent sequence of labels over $n$, each with equal probability $y_n=1$ or $y_n=-1$, and we also randomly generate 50 independent instances $\mathbf{x}_{n}\in\mathbb{R}^{300}$ (to pair with labels) from a uniform distribution denoted by $U(a,b)$. For each  $\mathbf{x}_{n}\in N_m$, the entries in the coordinates from 
$50m-49$ to $50m$ follow $U(0,1)$, those in the coordinates from $251$ to $300$ follow $U(0,10)$ and the other entries follow $U(0,0.01)$. This corresponds to a multi-agent system where each agent \textcolor[rgb]{0.00,0.00,0.00}{observes some specific} features and all agents have some common features. The step size $\alpha=0.0078$ is tuned for GD, and all other algorithms use the same step size except top-j and NoUnif-IAG. \textcolor[rgb]{0.00,0.00,0.00}{In top-$j$, we have tried all combinations for $j=\{1,10,100\}$ and $\gamma_0=\{0.01,0.1,1,10\}$ and found that setting $j=10$ and $\gamma_0=0.01$ achieves the best convergence performance}. \textcolor[rgb]{0.00,0.00,0.00}{We set the step size $\alpha'=\alpha/M$ for NoUnif-IAG to guarantee stability. } The parameters $\xi$ and $\tilde\xi$ are tuned to $\xi/M=80$ and $\tilde\xi/M=40$ that are the best for \textcolor[rgb]{0.00,0.00,0.00}{GD-SEC} and CGD, respectively. \textcolor[rgb]{0.00,0.00,0.00}{As shown in Fig. \ref{conv_fig2}, \textcolor[rgb]{0.00,0.00,0.00}{GD-SEC} is able to significantly reduce the total number of transmitted bits, while maintaining a fast convergence when compared to other algorithms. Given an objective error of $10^{-10}$, \textcolor[rgb]{0.00,0.00,0.00}{GD-SEC} can save about $91.22\%$ of the all transmitted bits compared to classical GD.}

\subsection{Impact of Error Correction}\label{errorCorrImpact}

\begin{figure}[!t]
\centering
\includegraphics[width=3.6in]{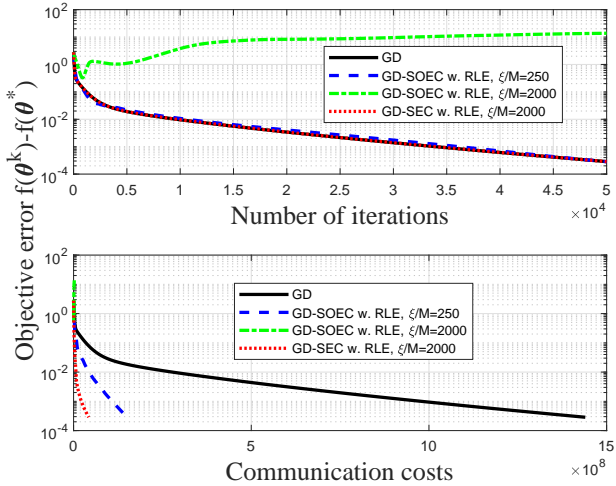}
\caption{Objective error versus the number of iterations and the communication costs for lasso regression in the DNA dataset.}
\label{conv_fig3}
\end{figure}

Next we consider \textcolor[rgb]{0.00,0.00,0.00}{the} impact of error correction for \textcolor[rgb]{0.00,0.00,0.00}{GD-SEC} on communication savings in \textcolor[rgb]{0.00,0.00,0.00}{a lasso regression using} the DNA dataset \cite{chang2011libsvm}.  \textcolor[rgb]{0.00,0.00,0.00}{Here}, the local function defined in (\ref{basicprobNSF2020}) for worker $m$ is
\begin{align}\label{lassoloss}
f_{m}(\boldsymbol{\theta}) :=\frac{1}{2N}\sum_{n=1}^{{N}_{m}}\left(y_{n}-\mathbf{x}_{n}^{\top} \boldsymbol{\theta}\right)^{2} + \frac{\lambda}{M}\|\boldsymbol{\theta}\|_1,
\end{align}
where $\|\boldsymbol{\theta}\|_1$ is $\ell_1$-norm of $\boldsymbol{\theta}$. Since $\|\boldsymbol{\theta}\|_1$ is nondifferentiable, worker $m$ computes the subgradient \textcolor[rgb]{0.00,0.00,0.00}{given by}
\begin{align}\label{lassolossGrad}
\partial f_{m}(\boldsymbol{\theta}) = -\frac{1}{N}\sum_{n=1}^{{N}_{m}}{\mathbf{x}_{n}}\left(y_{n}-\mathbf{x}_{n}^{\top} \boldsymbol{\theta}\right) + \frac{\lambda}{M}\mbox{sign}(\boldsymbol{\theta}),
\end{align}
where the sign function $\mbox{sign}(\boldsymbol{\theta})$ is defined as an element-wise operator. The step size is tuned to $\alpha=0.001$ for GD, and \textcolor[rgb]{0.00,0.00,0.00}{others} employs the same step size. The parameter $\xi$ is tuned to $\xi/M=2000$ for \textcolor[rgb]{0.00,0.00,0.00}{GD-SEC} and $\xi/M=250$ for \textcolor[rgb]{0.00,0.00,0.00}{GD with sparsification but without error correction\footnote{\textcolor[rgb]{0.00,0.00,0.00}{GD-SOEC can be obtained by setting $e_m^k=\boldsymbol{0}$ for every $m$ and $k$ in Algorithm \ref{varianceReduceOGD}.}}(GD-SOEC)} \textcolor[rgb]{0.00,0.00,0.00}{to} ensure these two proposed algorithms converge to the optimal solution \textcolor[rgb]{0.00,0.00,0.00}{with} the smallest number of transmitted bits. \textcolor[rgb]{0.00,0.00,0.00}{Note that both \textcolor[rgb]{0.00,0.00,0.00}{GD-SEC and GD-SOEC} use RLE.} \textcolor[rgb]{0.00,0.00,0.00}{We find that \textcolor[rgb]{0.00,0.00,0.00}{GD-SEC} converges and works well with a large range of $\xi/M\in(0,2000]$}, which is desirable in practice. Fig. \ref{conv_fig3} shows that \textcolor[rgb]{0.00,0.00,0.00}{with proper choice of $\xi/M$ both {GD-SEC} and {GD-SOEC}  can significantly reduce the number of transmitted bits compared to GD, and {GD-SEC} is the best overall. This implies we can set a large threshold in GD-SEC and even though many components are not transmitted, error correction mitigates the error due to this. This is not possible with GD-SOEC. }

\subsection{Impact of \textcolor[rgb]{0.00,0.00,0.00}{State Variable}}\label{effectofbeta}

\begin{figure}[!t]
\centering
\includegraphics[width=3.6in]{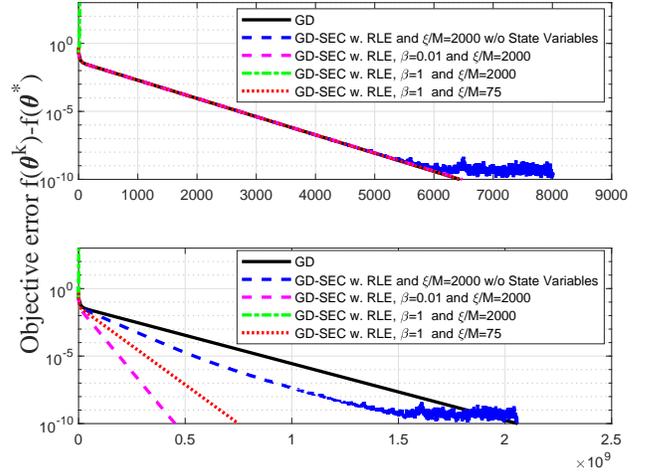}
\caption{Objective error versus the number of iterations and the communication costs for linear regression in the COLON-CANCER dataset.}
\label{conv_fig4}
\end{figure}

To understand how the weighted average of the previously transmitted gradients (\textcolor[rgb]{0.00,0.00,0.00}{state variable}) impacts the communication savings, we test the performance of \textcolor[rgb]{0.00,0.00,0.00}{GD-SEC} on the real dataset COLON-CANCER \cite{chang2011libsvm} for linear regression with different values of $\beta$. We tune the step size $\alpha= 1/L=0.0031$ for GD and let \textcolor[rgb]{0.00,0.00,0.00}{GD-SEC} use the same step size. \textcolor[rgb]{0.00,0.00,0.00}{Compared to GD-SEC without state variables, Fig. \ref{conv_fig4} shows that GD-SEC with state variables can save more transmitted bits without sacrificing convergence {speed}.  Even though some workers do not transmit in GD-SEC with state variables, the server can use its state variable to approximate the sum of the current local gradients of workers instead of just dropping them like GD-SEC without state variables.} As depicted in Fig. \ref{conv_fig4}, setting a relatively small $\beta$, e.g. $\beta=0.01$, enables us to choose a relatively large threshold $\xi$, e.g. $\xi/M=2000$, to reduce the number of transmitted bits. \textcolor[rgb]{0.00,0.00,0.00}{This makes sense {because} with a small $\beta>0$,  the overall descent direction smoothing by ${\boldsymbol{h}}_m^{k+1}$ {maintains accuracy} even though we suppress many components 
by setting a large $\xi/M$.} Fig. \ref{conv_fig4} also shows that increasing $\beta$ without decreasing $\xi$ might lead \textcolor[rgb]{0.00,0.00,0.00}{GD-SEC} to diverge. \textcolor[rgb]{0.00,0.00,0.00}{This behavior follows because increasing $\beta$ will make ${\boldsymbol{h}}_m^{k+1}$ depend more on the last transmitted gradient and be more sensitive to the gradient change, for example the state variable ${\boldsymbol{h}}_m^{k+1}$ reduces to the last transmitted gradient $\nabla f_m({\boldsymbol{\theta}}^{k})$ if we set $\beta=1$.} Thus, the results in Fig. \ref{conv_fig4} suggest that {for the given $\alpha=0.0031$}, in order to choose a large threshold $\xi/M$ to reduce the number of transmitted bits, setting a small $\beta$ with $\beta>0$, e.g. $\beta=0.01$, \textcolor[rgb]{0.00,0.00,0.00}{is} a good choice.

\subsection{Impact of $\{\xi_i\}_{i=1}^d$}

\begin{figure}[!t]
\centering
\includegraphics[width=3.5in]{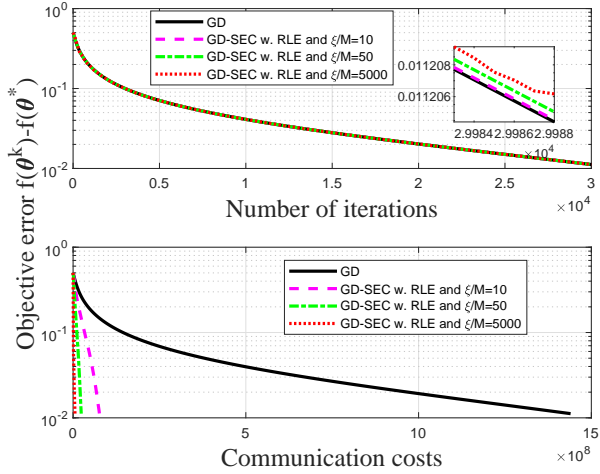}
\caption{Objective error versus the number of iterations and the communication costs for nonlinear least square in the W2A dataset.}
\label{conv_fig5}
\end{figure}

Next we evaluate the impact of the parameters $\{\xi_d\}_{i=1}^d$ on the communication savings of \textcolor[rgb]{0.00,0.00,0.00}{GD-SEC} in the W2A dataset \cite{chang2011libsvm} \textcolor[rgb]{0.00,0.00,0.00}{with a nonlinear least squares cost function}. For simplicity, we consider a special case where we set $\xi_1=\xi_2=...=\xi$. The local function defined in (\ref{basicprobNSF2020}) for worker $m$ for $m=1,2,...,M$ is \cite{xu2020second}
\begin{align}\label{nonlinearloss}
f_{m}(\boldsymbol{\theta}) :=\frac{1}{2N}\sum_{n=1}^{{N}_{m}}\left(y_{n}-\frac{1}{1+\exp(-\mathbf{x}_{n}^{\top} \boldsymbol{\theta})} \right)^{2} + \frac{\lambda}{2M}\|\boldsymbol{\theta}\|^2,
\end{align}
where $\lambda=1/N$. Note that $f_{m}(\boldsymbol{\theta})$ is non-convex. The step size is tuned to $\alpha=0.005$ for GD, and \textcolor[rgb]{0.00,0.00,0.00}{GD-SEC} employs the same step size. As shown in Fig. \ref{conv_fig5}, \textcolor[rgb]{0.00,0.00,0.00}{GD-SEC} with different $\xi$ performs similarly with GD in terms of the number of iterations but can significantly reduce the number of communication bits to achieve a given objective error. When $\xi$ increases, the number of transmitted bits are reduced at the cost of slightly increasing the number of iterations. Compared to GD, \textcolor[rgb]{0.00,0.00,0.00}{GD-SEC} with $\xi/M=5000$ \textcolor[rgb]{0.00,0.00,0.00}{employs only $0.38\%$ of the} number of bits to achieve the objective error $0.0112$. The results in Fig. \ref{conv_fig5} also indicate that \textcolor[rgb]{0.00,0.00,0.00}{GD-SEC} achieves a favorable communication-computation tradeoff through \textcolor[rgb]{0.00,0.00,0.00}{adaptive} sparsification, error correction and state variables.

\subsection{Different Coordinate-wise Lipschitz Constants}

\begin{figure}[!t]
\centering
\includegraphics[width=3.6in]{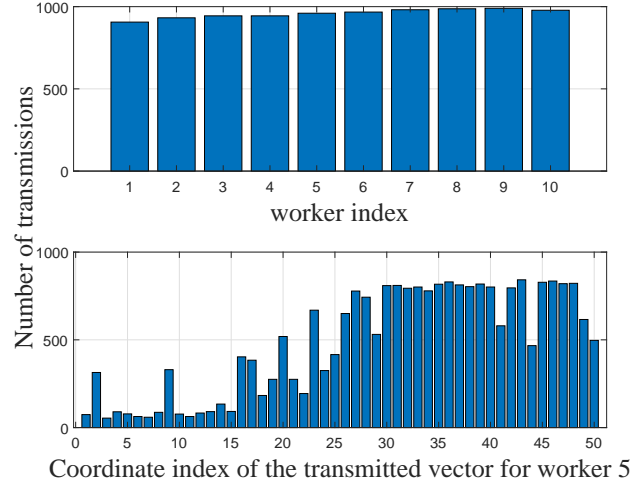}
\caption{Number of transmissions of different workers and different coordinates for the first 1000 iterations in linear regression with smoothness constants satisfying $L_1<L_2<...<L_{10}$ and coordinate-wise smoothness constants satisfying $L_m^1<L_m^2<...<L_m^{50}$ for $m=1,2,...,10$.}
\label{conv_fig6}
\end{figure}

{Now} we consider the impact of different smoothness constants at different coordinates on communication savings in a synthetic dataset for linear regression. Specifically, we consider regularized linear regression in (\ref{linearloss}) with $\lambda=0$ and the increasing coordinate-wise smoothness \textcolor[rgb]{0.00,0.00,0.00}{constants} with $L_m^1<L_m^2<...<L_m^{50}$ for worker $m$ in a scenario with a server and ten workers. For each worker $m$, we randomly generate an independent sequence of labels, each with equal probability \textcolor[rgb]{0.00,0.00,0.00}{of} $y_n=1$ or $y_n=-1$ for $n=1,2,...,50$, and we also randomly generate 50 independent instances  $\mathbf{x}_{n}\in\mathbb{R}^{50}$ (to pair with the labels) from a uniform distribution $U(0,0.01)$ but \textcolor[rgb]{0.00,0.00,0.00}{we replace the $n$-th entry of $\mathbf{x}_{n}$ with} $m\times1.1^n$ to mimic increasing coordinate-wise smoothness constants as the coordinate indices are increased. Note that in this setting the smoothness constants for different workers satisfy $L_1<L_2< ...<L_{10}$. The linear regression task is stopped after $1000$ iterations with the step size $\alpha=1/L=9.4241\times10^{-5}$ and the threshold $\xi=50000$. As indicated in Fig. \ref{conv_fig6}, \textcolor[rgb]{0.00,0.00,0.00}{workers with smaller smoothness constants} tend to transmit less frequently, which is reasonable since the smoothness constant measures the maximum rate of change of the gradient for worker $m$. The result in Fig. \ref{conv_fig6} also shows that for a given worker \textcolor[rgb]{0.00,0.00,0.00}{(e.g. $m=5$)}, the coordinates with smaller coordinate-wise smoothness constants tend to transmit less frequently which \textcolor[rgb]{0.00,0.00,0.00}{also} makes sense. The results in Fig. \ref{conv_fig6} imply that \textcolor[rgb]{0.00,0.00,0.00}{GD-SEC} can not only exploit the smoothness of the local functions for different workers but also capture the smoothness of different coordinates of the local function for a given worker so as to \textcolor[rgb]{0.00,0.00,0.00}{reduce the number of bits per transmission and the number of communications from all workers}.

{Inspired by the impact of different smoothness constants on different coordinates on communication savings in Fig. \ref{conv_figscale}, we consider further improving communication savings by scaling each component of $\xi$ differently using $\xi_i=\xi/L^i$ for $i=1,2,...,d$ with $L^i$ being the smoothness constant for the $i$-th coordinate of $f(\boldsymbol{\theta})$. Now we evaluate the overall impact of this setting on the communication savings of \textcolor[rgb]{0.00,0.00,0.00}{GD-SEC} {for} a sparse dataset called the RCV1-train dataset \cite{lewis2004rcv1} with {the} regularized logistic regression cost function shown in (\ref{logisticloss}). Here we employ $75\%$ of its data samples to train the model (in other words, we employ $15181$ data samples) and each data sample contains $47236$ features. We consider a scenario with one server and five workers and set the regularization parameter $\lambda=1/N$ with $N=15181$. We have performed a grid search for hyper-parameter tuning where we tried all the combinations for $\alpha\in\{2^{-10},2^{-9},...,2^{11}\}$, $\beta\in\{2^{-10},2^{-9},...,2^{0}\}$ and $\xi\in\{2^{-5},2^{-4},...,2^{14}\}$. For a given objective function value, we choose the best $\alpha$ for GD, and choose the best $\alpha$, $\beta$ and $\xi$ for \textcolor[rgb]{0.00,0.00,0.00}{GD-SEC}. In Fig. \ref{conv_figscale}, we plot the objective function value versus the total number of transmitted entries in the first $1000$ iterations. The results in Fig. \ref{conv_figscale} illustrate that setting $\xi_i=\xi/L^i$ does lead to larger communication savings compared to setting $\xi_i=\xi$ for all $i=1,2,...,47236$, which makes sense because the gradient in the coordinate with a smaller $L^i$ tends to change less frequently and thus we can use a larger threshold by setting $\xi_i=\xi/L^i$ to reduce the number of transmitted entries per communication. }

\begin{figure}[!t]
\centering
\includegraphics[width=3.5in]{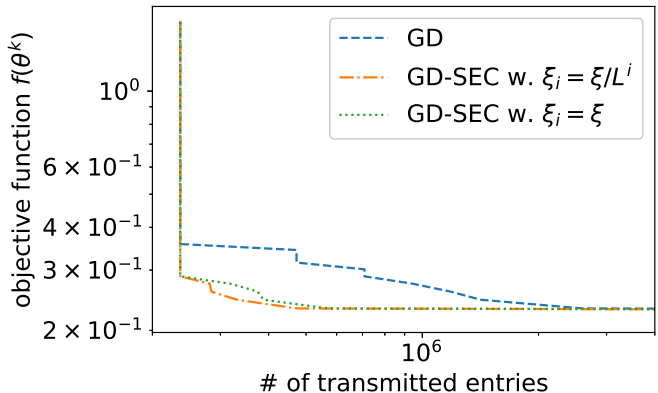}
\caption{Objective function value versus the total number of transmitted entries for logistic regression in the RCV1-train
dataset.}
\label{conv_figscale}
\end{figure}

\textcolor[rgb]{0.00,0.00,0.00}{\subsection{Possible Extensions of GD-SEC}\label{extendSec}}

\textcolor[rgb]{0.00,0.00,0.00}{In the following subsections, we consider two possible extensions to satisfy some limitations in \textcolor[rgb]{0.00,0.00,0.00}{practice. First, allowing a portion of workers for parameter uploading at each iteration,  and second,} developing the stochastic version GD-SEC (SGD-SEC). We apply the RLE algorithm in GD-SEC and SGD-SEC to encode the indices of transmitted non-zero components. We do not apply RLE to quantized SGD-SEC. }

\subsubsection{Bandwidth-limited Version of GD-SEC}

\begin{figure}[!t]
\centering
\includegraphics[width=3.6in]{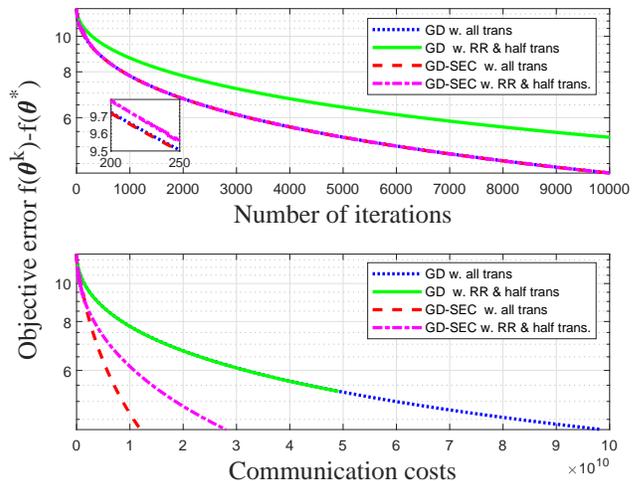}
\caption{\textcolor[rgb]{0.00,0.00,0.00}{Objective error versus the number of iterations and the communication costs for linear regression in the {CIFAR}-10 dataset  (using 2000 data samples).}}
\label{conv_figPortion}
\end{figure}

\textcolor[rgb]{0.00,0.00,0.00}{Motivated by limited spectral resources and unreliable clients, we consider a scenario where the server cannot collect updates from all the workers at each iteration and instead only schedules a portion of workers for parameter uploading. A bandwidth-limited version of GD-SEC can be obtained by only allowing a portion of workers to transmit.\footnote{Specifically, we need to change line 3 in Algorithm \ref{varianceReduceOGD} to \textbf{for} $m\in\mathcal{M}_k$ \textbf{do} where $\mathcal{M}_k$ is a set of selected worker indices at iteration $k$. $\mathcal{M}_k$ can be obtained using the scheduling policies in \cite{yang2019scheduling}.} In Fig. \ref{conv_figPortion}, we consider a setting with one server and 100 workers where each worker has its own local linear regression function defined in (\ref{linearloss}). We randomly choose 2000 data samples from the standardized \textcolor[rgb]{0.00,0.00,0.00}{CIFAR}-10 dataset and evenly split them \textcolor[rgb]{0.00,0.00,0.00}{among the} 100 workers. We set $\lambda=1/N$ and tune the step size $\alpha= 2/L\approx0.0023$ for all considered algorithms. One scheduling policy called round robin (RR) in \cite{yang2019scheduling} is considered here to explore how only using half of {the} workers per iteration affects the performance of GD-SEC. We tried $\xi/M=\{1,10,100\}$ for GD-SEC related algorithms. Specifically we set $\xi/M=100$ for GD-SEC with all transmissions per iteration and set $\xi/M=10$ for GD-SEC with RR and half transmissions per iteration. Note that setting $\xi/M=100$ for the latter case leads to divergence in this scenario which makes sense {because} setting a larger threshold and only allowing {half the} workers to transmit might not enable the server to have enough correct information about function value descent direction. It is not surprising that GD with half transmissions per iteration has worse convergence behavior compared to
GD with all transmissions per iteration. {Interestingly}, Fig.   \ref{conv_figPortion} indicates GD-SEC with RR and half transmissions per iteration makes progress which is only slightly slower than GD-SEC and GD with all transmissions per iteration {for this example}. Even though some workers do not transmit, the server can still use its state variable to approximate the sum of the current local gradients of workers.    Compared with GD with all transmissions, Fig. \ref{conv_figPortion} also shows that the GD-SEC related algorithms significantly outperform classical GD in terms of communication bit savings, and GD-SEC with all workers transmitting per iteration achieves the best performance in terms of transmitted bit savings. \textcolor[rgb]{0.00,0.00,0.00}{Although beyond the scope of this work, obtaining an optimal scheduling policy for GD-SEC is an interesting open question for future work.} }

\subsubsection{Stochastic Version of GD-SEC}

\textcolor[rgb]{0.00,0.00,0.00}{Now we consider a stochastic version of GD-SEC (SGD-SEC) where each local gradient is computed based on randomly selected samples of the local dataset instead of the entire local dataset. To further improve the communication bit savings of SGD-SEC, we propose a quantized SGD-SEC (QSGD-SEC) by combining SGD-SEC with a widely used low precision unbiased quantizer \cite{alistarh2017qsgd,reisizadeh2020fedpaq}. We employ the same setting as Section \ref{linearReMNIST} but we split {$N=6000$} between {$M=100$} where the {mini-batch size} is set to be one. We also set $\lambda=1/N$, $\alpha_k=\gamma_0(1+\gamma_0\lambda k)^{-1}$ with constants   $\gamma_0=0.01$ for all considered algorithms, and $\xi/M=100$ for SGD-SEC related algorithms. All other parameters are the same as Section \ref{linearReMNIST}. Fig. \ref{conv_figstochastic} shows that SGD-SEC outperforms SGD in reducing the number of transmitted bits while achieving nearly the same convergence. It also shows that the performance of SGD-SEC in terms of bit savings can be further improved using quantization {of} the non-zero components after sparsification.  }


\begin{figure}[!t]
\centering
\includegraphics[width=3.6in]{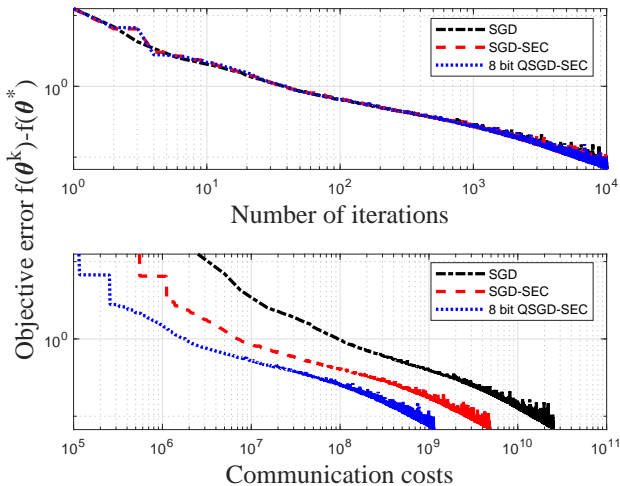}
\caption{\textcolor[rgb]{0.00,0.00,0.00}{Objective error versus the number of iterations and the communication costs for linear regression in the MNIST dataset  (using 6000 data samples).}}
\label{conv_figstochastic}
\end{figure}

\section{Conclusions}\label{Conclusion}

We proposed \textcolor[rgb]{0.00,0.00,0.00}{GD-SEC}, a communication-efficient \textcolor[rgb]{0.00,0.00,0.00}{enhanced} version of the classical GD method, to solve distributed learning problems. \textcolor[rgb]{0.00,0.00,0.00}{GD-SEC} is able to significantly reduce \textcolor[rgb]{0.00,0.00,0.00}{the number of transmitted bits} compared to GD without sacrificing convergence, 
{which is suitable for systems where wireless communication cost is the bottleneck on overall performance.} There are three important ingredients to \textcolor[rgb]{0.00,0.00,0.00}{attaining} this favorable performance, adaptive sparsification, error correction, and state variables. Adaptive sparsification prevents each  worker from transmitting less informative entries of the gradient and so reduces the number of transmitted bits. In order to \textcolor[rgb]{0.00,0.00,0.00}{avoid} the case where certain entries are never updated,  error correction \textcolor[rgb]{0.00,0.00,0.00}{is} employed to accumulate the error from the past and add it back to the current gradient before sparsification. A state variable was used at each worker to locally smooth the previously transmitted gradients to make \textcolor[rgb]{0.00,0.00,0.00}{GD-SEC} \textcolor[rgb]{0.00,0.00,0.00}{more smoothly} follow the consistent descent direction.
We have \textcolor[rgb]{0.00,0.00,0.00}{shown} theoretically that \textcolor[rgb]{0.00,0.00,0.00}{GD-SEC will maintain the same order of the convergence rate as GD for strongly convex, convex, and nonconvex objective functions. Numerical results validate the convergence and significant communication bit savings of GD-SEC.} \textcolor[rgb]{0.00,0.00,0.00}{One possible future direction is to combine GD-SEC with parameter pruning to further reduce the communication {load}. \textcolor[rgb]{0.00,0.00,0.00}{We also showed by example that a stochastic version (SGD-SEC) can perform well. } }


\appendices

\section{ }\label{prooflemma1}
In the appendix, we provide the proofs of the lemmas and theorems in the main document. Before we get into the detailed proofs, we first establish the following descent lemma of \textcolor[rgb]{0.00,0.00,0.00}{GD-SEC} which will help us develop the proof of Lemma \ref{descentLyapunovfunv}.

\begin{lemma}\label{descent}
Under Assumption \ref{OHBassumption1}, suppose $\boldsymbol{\theta}_{k+1}$ is generated by running one-step \textcolor[rgb]{0.00,0.00,0.00}{GD-SEC} iteration given $\boldsymbol{\theta}_{k}$. Then the objective function satisfies
\begin{align}\label{descentfunct}
 &f(\boldsymbol{\theta}^{k+1})- f(\boldsymbol{\theta}^{k}) \le \frac{\alpha}{2}\left\| \sum_{m=1}^M \left( \hat{\Delta}_m^k - {\Delta}_m^k + {\boldsymbol{e}}_m^k \right) \right\|^{2}\notag\\
 &\quad-\frac{\alpha}{2}\left\|\nabla f\left(\boldsymbol{\theta}^{k}\right)\right\|^{2}+\left(\frac{L}{2}-\frac{1}{2 \alpha}\right)\left\|\boldsymbol{\theta}^{k+1}-\boldsymbol{\theta}^{k}\right\|^{2}.
\end{align}
\end{lemma}

\begin{proof}
Recall the smoothness constant $L$ defined in Assumption \ref{OHBassumption1}.  One property of the $L$-smooth function $f(\boldsymbol{\theta})$ is that \cite{nesterov2013introductory}
\begin{align}\label{smothfun}
f\left(\boldsymbol{\theta}^{k+1}\right)-f\left(\boldsymbol{\theta}^{k}\right) &\leq\left\langle\nabla f\left(\boldsymbol{\theta}^{k}\right), \boldsymbol{\theta}^{k+1}-\boldsymbol{\theta}^{k}\right\rangle\notag\\
&+\frac{L}{2}\left\|\boldsymbol{\theta}^{k+1}-\boldsymbol{\theta}^{k}\right\|^{2}.
\end{align}
Plugging (\ref{updateruleSpar}) into $\left\langle\nabla f\left(\boldsymbol{\theta}^{k}\right), \boldsymbol{\theta}^{k+1}-\boldsymbol{\theta}^{k}\right\rangle$, it follows that
\begin{align}
&\left\langle\nabla f\left(\boldsymbol{\theta}^{k}\right), \boldsymbol{\theta}^{k+1}-\boldsymbol{\theta}^{k}\right\rangle\nonumber\\
&\stackrel{\textcolor[rgb]{0.00,0.00,0.00}{(\ref{updateruleSpar})}}{=}-\alpha\left\langle\nabla f\left(\boldsymbol{\theta}^{k}\right),  \sum_{m=1}^M\left(\hat{\Delta}_m^k + {\boldsymbol{h}}_m^{k}\right)
\right\rangle\nonumber\\
&=-\alpha\left\langle\nabla f\left(\boldsymbol{\theta}^{k}\right), \nabla f\left(\boldsymbol{\theta}^{k}\right) +  \sum_{m=1}^M\left(\hat{\Delta}_m^k -{\Delta}_m^k + {\boldsymbol{e}}_m^k \right)
\right\rangle\nonumber\\
&=-\alpha\left\|\nabla f\left(\boldsymbol{\theta}^{k}\right)\right\|^2 \notag\\
&\qquad+ \alpha \left\langle\nabla f\left(\boldsymbol{\theta}^{k}\right), \sum_{m=1}^M\left(-\hat{\Delta}_m^k +{\Delta}_m^k - {\boldsymbol{e}}_m^k \right)
\right\rangle\label{laststepinner}\\
&=-\frac{\alpha}{2}\left\| \nabla f\left(\boldsymbol{\theta}^{k}\right)\right\|^2 + \frac{\alpha}{2}\left\| \sum_{m=1}^M \left( \hat{\Delta}_m^k - {\Delta}_m^k + {\boldsymbol{e}}_m^k \right)\right\|^2\notag\\
&\qquad-\frac{\alpha}{2}\left\| \nabla f\left(\boldsymbol{\theta}^{k}\right) +  \sum_{m=1}^M\left(\hat{\Delta}_m^k -{\Delta}_m^k + {\boldsymbol{e}}_m^k \right)   \right\|^2\label{innerprodemix}\\
&=-\frac{\alpha}{2}\left\| \nabla f\left(\boldsymbol{\theta}^{k}\right)\right\|^2 + \frac{\alpha}{2}\left\|\sum_{m=1}^M \left( \hat{\Delta}_m^k - {\Delta}_m^k + {\boldsymbol{e}}_m^k\right) \right\|^2\notag\\
&\qquad-\frac{1}{2\alpha}\left\|\boldsymbol{\theta}^{k+1}-\boldsymbol{\theta}^{k}
\right\|^2.\label{breakinnerpro}
\end{align}
In going from (\ref{laststepinner}) to (\ref{innerprodemix}), we use the equality $\left\langle \boldsymbol x, \boldsymbol y\right\rangle=\frac{1}{2}\|\boldsymbol x\|^2+\frac{1}{2}\|\boldsymbol y\|^2-\frac{1}{2}\|\boldsymbol{x-y}\|^2$ for the last term in (\ref{laststepinner}). The result in Lemma \ref{descent} follows after plugging (\ref{breakinnerpro}) into (\ref{smothfun}).
\end{proof}

\section{Proof of Lemma \ref{descentLyapunovfunv}}\label{decreaseproperty}

\begin{proof}

Using the definition of $\mathbb{L}^{k}$ in (\ref{Laypufunc}), it follows that
\begin{align}
&\mathbb{L}^{k+1}-\mathbb{L}^{k}\notag\\
&=f\left(\boldsymbol{\theta}^{k+1}\right)-f\left(\boldsymbol{\theta}^{k}\right)+\beta_1\left\|\boldsymbol{\theta}^{k+1}-\boldsymbol{\theta}^{k}\right\|^2 \notag\\
&\qquad+(\beta_2-\beta_1)\left\|\boldsymbol{\theta}^{k}-\boldsymbol{\theta}^{k-1}\right\|^2- \beta_2\left\|\boldsymbol{\theta}^{k-1}-\boldsymbol{\theta}^{k-2}\right\|^2\nonumber\\
&\stackrel{\textcolor[rgb]{0.00,0.00,0.00}{(\ref{descentfunct})}}{\le}-\frac{\alpha}{2}\left\|\nabla f\left(\boldsymbol{\theta}^{k}\right)\right\|^{2}+\frac{\alpha}{2}\left\| \sum_{m=1}^M \left( \hat{\Delta}_m^k - {\Delta}_m^k + {\boldsymbol{e}}_m^k \right) \right\|^{2}\notag\\
&\qquad+\left(\frac{L}{2}-\frac{1}{2 \alpha} + \beta_1\right)\left\|\boldsymbol{\theta}^{k+1}-\boldsymbol{\theta}^{k}\right\|^{2}\nonumber\\
&\qquad+(\beta_2-\beta_1)\left\|\boldsymbol{\theta}^{k}-\boldsymbol{\theta}^{k-1}\right\|^2- \beta_2\left\|\boldsymbol{\theta}^{k-1}-\boldsymbol{\theta}^{k-2}\right\|^2\label{plugdescentinequality}.
\end{align}

\textcolor[rgb]{0.00,0.00,0.00}{After plugging $\Delta_m^k:=\nabla f_m({\boldsymbol{\theta}}^k)-{\boldsymbol{h}}_m^k+{\boldsymbol{e}}_m^k$ into (\ref{updateruleSpar}),  the \textcolor[rgb]{0.00,0.00,0.00}{GD-SEC} update rule can be rewritten as
\begin{align}\label{updategradientpara}
\boldsymbol{\theta}^{k+1}
&= \boldsymbol{\theta}^{k}-\alpha\nabla f\left(\boldsymbol{\theta}^{k}\right) - {\alpha} \sum_{m=1}^M\left(\hat{\Delta}_m^k -{\Delta}_m^k + {\boldsymbol{e}}_m^k \right),
\end{align}
where $\nabla f(\boldsymbol{\theta}^{k})$ is the gradient of the objective function in (\ref{basicprobNSF2020}).}

From (\ref{updategradientpara}), we have
\begin{align}
&\left\|\boldsymbol{\theta}^{k+1}-\boldsymbol{\theta}^{k}\right\|^{2}\notag\\
&=\left\| -\alpha\nabla f\left(\boldsymbol{\theta}^{k}\right) - {\alpha} \sum_{m=1}^M\left(\hat{\Delta}_m^k -{\Delta}_m^k + {\boldsymbol{e}}_m^k \right)  \right\|^{2}\label{younfqin0}\\
&\le(1+\rho)\alpha^2\left\|  \nabla f\left(\boldsymbol{\theta}^{k}\right) \right\|^{2}\notag\\
&\quad\quad + (1+\rho^{-1})\alpha^2\left\|   \sum_{m=1}^M\left(\hat{\Delta}_m^k -{\Delta}_m^k + {\boldsymbol{e}}_m^k \right) \right\|^{2}.\label{younfqin1}
\end{align}
In going from (\ref{younfqin0}) to (\ref{younfqin1}), we employ Young's inequality: $\|\boldsymbol{x+y}\|^2\le(1+\rho)\|\boldsymbol{x}\|^2+(1+\rho^{-1})\|\boldsymbol{y}\|^2, \forall\rho>0$. Plugging (\ref{younfqin1}) into (\ref{plugdescentinequality}) with the requirement $(\frac{L}{2}-\frac{1}{2 \alpha} + \beta_1)\ge0$, we obtain
\begin{align}
&\mathbb{L}^{k+1}-\mathbb{L}^{k}\notag\\
&\le\left(-\frac{\alpha}{2} + (\frac{L}{2}-\frac{1}{2 \alpha}+\beta_1)(1+\rho)\alpha^2  \right)\left\|\nabla f\left(\boldsymbol{\theta}^{k}\right)\right\|^{2}\nonumber\\
&\quad+\left( \frac{\alpha}{2} + (\frac{L}{2}-\frac{1}{2 \alpha}+\beta_1)(1+\rho^{-1})\alpha^2 \right)\Big\Vert \sum_{m=1}^M\big(\hat{\Delta}_m^k \notag\\
&\quad\quad\quad-{\Delta}_m^k + {\boldsymbol{e}}_m^k \big) \Big\Vert ^{2}\nonumber\\
&\quad+(\beta_2-\beta_1)\left\|\boldsymbol{\theta}^{k}-\boldsymbol{\theta}^{k-1}\right\|^2- \beta_2\left\|\boldsymbol{\theta}^{k-1}-\boldsymbol{\theta}^{k-2}\right\|^2.\label{Laydesc}
\end{align}
Employing Young's inequality again, it follows that
\begin{align}
&\left\|  \sum_{m=1}^M\left(\hat{\Delta}_m^k -{\Delta}_m^k + {\boldsymbol{e}}_m^k \right) \right\|^{2}\notag\\
&\le (1+\rho_2)\left\| \sum_{m=1}^M\left(\hat{\Delta}_m^k -{\Delta}_m^k  \right) \right\|^{2} + (1+\rho_2^{-1})\left\| \sum_{m=1}^M {\boldsymbol{e}}_m^k  \right\|^{2}\\
&\le (1+\rho_2)\xi^2\left\| \boldsymbol{\theta}^{k}-\boldsymbol{\theta}^{k-1} \right\|^2 + (1+\rho_2^{-1})\xi^2\left\| \boldsymbol{\theta}^{k-1}-\boldsymbol{\theta}^{k-2} \right\|^2,\label{thresholdupdating}
\end{align}
where $\xi=\max_i\xi_i$ and (\ref{thresholdupdating}) is obtained using the \textcolor[rgb]{0.00,0.00,0.00}{GD-SEC}-sparsify-transmission condition in (\ref{ASGD-stopping condition}).
{Plugging} (\ref{thresholdupdating}) into (\ref{Laydesc}) with $\frac{L}{2}-\frac{1}{2 \alpha}+\beta_1>0$, we have
\begin{align}\label{decreasingpropertyofLay}
&\mathbb{L}^{k+1}-\mathbb{L}^{k}\notag\\
&\le\left(-\frac{\alpha}{2} + (\frac{L}{2}-\frac{1}{2 \alpha}+\beta_1)(1+\rho)\alpha^2  \right)\left\|\nabla f\left(\boldsymbol{\theta}^{k}\right)\right\|^{2}\nonumber\\
&\quad+\Bigg(  (1+\rho_2)\xi^2\Big( \frac{\alpha}{2} + (\frac{L}{2}-\frac{1}{2 \alpha}+\beta_1)(1+\rho^{-1})\alpha^2 \Big)\notag\\
&\qquad\qquad +\beta_2-\beta_1 \Bigg)\left\| \boldsymbol{\theta}^{k}-\boldsymbol{\theta}^{k-1} \right\|^2\notag\\
&\quad+\Bigg( (1+\rho_2^{-1})\xi^2\left( \frac{\alpha}{2} + (\frac{L}{2}-\frac{1}{2 \alpha}+\beta_1)(1+\rho^{-1})\alpha^2 \right)\notag\\
&\qquad\qquad-\beta_2 \Bigg)\left\| \boldsymbol{\theta}^{k-1}-\boldsymbol{\theta}^{k-2} \right\|^2.
\end{align}
Furthermore, if the step size $\alpha$, the constants $\xi$ and $\beta$ satisfy
\begin{align}
\frac{L}{2}-\frac{1}{2 \alpha}+\beta_1&\ge 0,\label{conefpositveLay}\\
-\frac{\alpha}{2} + (\frac{L}{2}-\frac{1}{2 \alpha}+\beta_1)(1+\rho)\alpha^2&\le 0,\label{123conefnegaLay}\\
(1+\rho_2)\xi^2\Bigg( \frac{\alpha}{2} + (\frac{L}{2}-\frac{1}{2 \alpha}+\beta_1)(1+\rho^{-1})\alpha^2 \Bigg)&\notag\\
\qquad+\beta_2-\beta_1&\le 0,\\
(1+\rho_2^{-1})\xi^2\Big( \frac{\alpha}{2} + (\frac{L}{2}-\frac{1}{2 \alpha}+\beta_1)(1+\rho^{-1})\alpha^2 \Big)&\notag\\
-\beta_2&\le 0,\label{conefpositveLayend}
\end{align}
then Lyapunov function is non-increasing; that is, $\mathbb{L}^{k+1}\le\mathbb{L}^{k}$. The proof is complete after defining the non-negative constants $\sigma_0$, $\sigma_1$ and $\sigma_2$ as shown in (\ref{DDScondition1})--(\ref{condition3}), respectively. Since the expressions of $\sigma_0$, $\sigma_2$ and $\sigma_2$ are complicated, we provide several choices of parameters in the following.

\paragraph{Choice of parameters.} The conditions (\ref{conefpositveLay})--(\ref{conefpositveLayend}) are equivalent to
\begin{align}
&\frac{1}{2\beta_1+L}\le\alpha\le\frac{2+\rho_1}{(2\beta_1+L)(1+\rho)}, \label{equalconditionsdic1}\\
& \xi\le \sqrt{\frac{2(\beta_1-\beta_2)}{(1+\rho_2)[\alpha+(\alpha L-1+2\alpha\beta_1)(1+\rho^{-1})\alpha]}}\label{equalconditionsdic2},\\
&\xi\le \sqrt{\frac{2\beta_2}{(1+\rho_2^{-1})[\alpha+(\alpha L-1+2\alpha\beta_1)(1+\rho^{-1})\alpha]}}\label{equalconditionsdic3}.
\end{align}

$\bullet$ If $\beta_1=\frac{1-\alpha L}{2\alpha}$, then (\ref{equalconditionsdic1})--(\ref{equalconditionsdic3}) is equivalent to
\begin{align}
\alpha&\le\frac{1}{L}, \quad \xi\le\min\Big\{\sqrt{\frac{2(\beta_1-\beta_2)}{(1+\rho_2)\alpha}}, \sqrt{\frac{2\beta_2}{(1+\rho_2^{-1})\alpha}}\Big\}.
\end{align}

$\bullet$ If $\beta_1\rightarrow0$, $\beta_2\rightarrow0$, $\rho\rightarrow0$, and $\rho_1\rightarrow0$, then (\ref{equalconditionsdic1})--(\ref{equalconditionsdic3}) degenerate to $1/L\le\alpha\le 2/L$ and $\xi=0$ which implies \textcolor[rgb]{0.00,0.00,0.00}{GD-SEC} reduces to the classical GD method in this case.

$\bullet$ If $\alpha=1/L$, then (\ref{equalconditionsdic1})--(\ref{equalconditionsdic3}) are equivalent to
\begin{align}
&0\le\beta_1\le\frac{L}{2(1+\rho)},\notag\\
&\xi\le\min\Bigg\{\sqrt{\frac{2L^2(\beta_1-\beta_2)}{(1+\rho_2)[L+2\beta_1(1+\rho^{-1})]}}, \notag\\ &\ \qquad\qquad\sqrt{\frac{2L^2\beta_2}{(1+\rho_2)[L+2\beta_1(1+\rho^{-1})]}}\Bigg\}.
\end{align}
\end{proof}

\section{Proof of Theorem \ref{stonglyconvextheo}}\label{linearstronglyconvex}

\begin{proof}
Using strong convexity in Assumption \ref{OHBassumption2}, $f(\boldsymbol{\theta})$ has the following property \cite{nesterov2013introductory}
\begin{align}\label{stonglyconvex}
2 \mu\left(f\left(\boldsymbol{\theta}^{k}\right)-f\left(\boldsymbol{\theta}^{*}\right)\right) \leq\left\|\nabla f\left(\boldsymbol{\theta}^{k}\right)\right\|^{2}.
\end{align}
Plugging (\ref{stonglyconvex}) into (\ref{decreasingpropertyofLay}) with $\sigma_0\ge0$ in (\ref{DDScondition1}) and $\beta_1\beta_2\ne0$, it follows that
\begin{align}
&\mathbb{L}^{k+1}-\mathbb{L}^{k}\notag\\
&\le -\left(\alpha\mu-(L-\frac{1}{\alpha}+2\beta_1)(1+\rho)\alpha^2\mu\right)\left(f\left(\boldsymbol{\theta}^{k}\right)-f\left(\boldsymbol{\theta}^{*}\right)\right)\nonumber\\
&\quad-\Bigg(- \frac{(1+\rho_2)\xi^2}{\beta_1}\Big( \frac{\alpha}{2} + (\frac{L}{2}-\frac{1}{2 \alpha}+\beta_1)(1+\rho^{-1})\alpha^2 \Big) \notag\\
& \qquad\qquad +  1 -\frac{\beta_2}{\beta_1} \Bigg)\beta_1\left\| \boldsymbol{\theta}^{k}-\boldsymbol{\theta}^{k-1} \right\|^2\\
&\quad-\Bigg(- \frac{(1+\rho_2^{-1})\xi^2}{\beta_2}\Big( \frac{\alpha}{2} + (\frac{L}{2}-\frac{1}{2 \alpha}+\beta_1)(1+\rho^{-1})\alpha^2 \Big) \notag\\
&\qquad\qquad+1\Bigg)\beta_2\left\| \boldsymbol{\theta}^{k-1}-\boldsymbol{\theta}^{k-2} \right\|^2.
\end{align}

Let us define $c\left(\alpha ;\xi\right)$ as
\begin{align}\label{coeffczeroone}
&c\left(\alpha ;\xi\right):=\min\Bigg\{\alpha\mu-(L-\frac{1}{\alpha}+2\beta_1)(1+\rho)\alpha^2\mu,\nonumber\\
& 1 -\frac{\beta_2}{\beta_1}- \frac{(1+\rho_2)\xi^2}{\beta_1}\left( \frac{\alpha}{2} + (\frac{L}{2}-\frac{1}{2 \alpha}+\beta_1)(1+\rho^{-1})\alpha^2 \right),\nonumber\\
& 1 - \frac{(1+\rho_2^{-1})\xi^2}{\beta_2}\left( \frac{\alpha}{2} + (\frac{L}{2}-\frac{1}{2 \alpha}+\beta_1)(1+\rho^{-1})\alpha^2 \right)  \Bigg\}\\
&=\min\Big\{ 2\sigma_0\mu, \sigma_1/\beta_1, \sigma_2/\beta_2\Big\},\label{coeffczerooneSimple}
\end{align}
where (\ref{coeffczerooneSimple}) follows using the definitions of $\sigma_0$, $\sigma_1$ and $\sigma_2$  in (\ref{DDScondition1})--(\ref{condition3}), respectively.
With the definition of $c\left(\alpha,\xi\right)$ in (\ref{coeffczeroone}), we obtain
\begin{align}
&\mathbb{L}^{k+1}-\mathbb{L}^{k}\notag\\
&\le -c\left(\alpha ;\xi\right) \Big(f\left(\boldsymbol{\theta}^{k}\right)-f\left(\boldsymbol{\theta}^{*}\right)+ \beta_1\left\|\boldsymbol{\theta}^{k}-\boldsymbol{\theta}^{k-1}\right\|^{2} \notag\\ &\qquad\qquad+\beta_2\left\|\boldsymbol{\theta}^{k-1}-\boldsymbol{\theta}^{k-2}\right\|^{2} \Big)\\
&=-c\left(\alpha ;\xi\right)\mathbb{L}^{k}
\end{align}
which implies
\begin{align}
\mathbb{L}^{k+1}
=\left(1-c\left(\alpha ;\xi\right)\right)\mathbb{L}^{k}.
\end{align}
This completes the proof. Now we consider a simple case where we let $\beta_1-\frac{1-\alpha L}{2\alpha}=0$ with $\beta_1\ne0$. Then $c\left(\alpha,\xi\right)$ defined in (\ref{coeffczeroone}) becomes
\begin{align}\label{constantcpara}
&c\left(\alpha,\xi\right)=\min\Bigg\{ {\alpha}\mu,
1 -\frac{\beta_2}{\beta_1}- \frac{(1+\rho_2)\alpha\xi^2}{2\beta_1},\notag\\
&\qquad\qquad\qquad\qquad 1 - \frac{(1+\rho_2^{-1})\alpha\xi^2}{2\beta_2}\Bigg\}.
\end{align}
If we choose $\rho_2=1$, $\delta\in(0,1)$, $\alpha=\frac{1-\delta}{L}$, $\xi^2\le\frac{1-\alpha\mu}{\alpha}$,
\begin{align}
 \beta_2=\frac{\alpha\xi^2}{1-\alpha\mu}, \quad
\mbox{and}\ \beta_1=\beta_2+\frac{1}{1-\alpha\mu},
\end{align}
then $c\left(\alpha,\xi\right)$ in (\ref{constantcpara}) becomes
\begin{align}
c\left(\alpha,\xi\right)&=\min\Bigg\{ \alpha\mu, 1-\frac{\alpha\xi^2(1+1/(1-\alpha\mu))}{\beta_1}\Bigg\}\notag\\
&=\alpha\mu.
\end{align}
This implies
\begin{align}\label{iterationcomp}
\frac{\mathbb{L}^{k+1}}{\mathbb{L}^{1}}\le\left(1- \frac{1-\delta}{L/\mu}\right)^k\le\epsilon,
\end{align}
\textcolor[rgb]{0.00,0.00,0.00}{where $\epsilon$ is a user defined error parameter with its value typically much less than one.}
After rearranging terms in (\ref{iterationcomp}), we can obtain
\begin{align}\label{}
\log(\frac{1}{\epsilon})\le k\log\left(1+\frac{1}{\frac{L/\mu}{1-\delta}-1}\right)\le  \frac{k}{\frac{L/\mu}{1-\delta}-1}
\end{align}
which implies the iteration complexity $\mathbb{I}_{\textcolor[rgb]{0.00,0.00,0.00}{GD-SEC}}(\epsilon)$ is
\begin{align}\label{}
\mathbb{I}_{\textcolor[rgb]{0.00,0.00,0.00}{GD-SEC}}(\epsilon)= \frac{L/\mu}{1-\delta}\log(\frac{1}{\epsilon}).
\end{align}
\textcolor[rgb]{0.00,0.00,0.00}{This indicates that for any given error $\epsilon$, there exists a large $k=\mathbb{I}_{\textcolor[rgb]{0.00,0.00,0.00}{GD-SEC}}(\epsilon)$ such that (\ref{iterationcomp}) is always satisfied.}
\end{proof}

\section{Proof of Theorem \ref{convextheo}}\label{convergeconv}
\begin{proof}
The analysis in this part is analogous to that in \cite{chen2018lag} that does not consider adaptive sparsification and error correction. Before providing the proof, we first present a lemma that will be helpful to derive the result in Theorem \ref{convextheo}.
\begin{lemma}\label{lyapunvfunctinequa}
Under \emph{Assumption \ref{OHBassumption3}}, the Lyapunov function $\mathbb{L}^{k}$ satisfies
\begin{align}
&\mathbb{L}^{k}\notag\\
&\le \sqrt{\Big(\|\nabla f(\boldsymbol{\theta}^k)\|^2 + \beta_{1}\|\boldsymbol{\theta}^{k}-\boldsymbol{\theta}^{k-1}\|^{2} + \beta_{2}\|\boldsymbol{\theta}^{k-1}-\boldsymbol{\theta}^{k-2}\|^{2} \Big)}\notag\\
&\sqrt{\Big(\|\boldsymbol{\theta}^k-\boldsymbol{\theta}^*\|^2 + \beta_{1}\|\boldsymbol{\theta}^{k}-\boldsymbol{\theta}^{k-1}\|^{2} + \beta_{2}\|\boldsymbol{\theta}^{k-1}-\boldsymbol{\theta}^{k-2}\|^{2}\Big)}.
\end{align}
\end{lemma}
\begin{proof}
Under {Assumption \ref{OHBassumption3}}, the objective function $ f(\boldsymbol{\theta})$ is convex which implies
\begin{align}\label{conveproer12}
f(\boldsymbol{\theta}^k)-f(\boldsymbol{\theta}^*)\le\left\langle\nabla f(\boldsymbol{\theta}^{k}), \boldsymbol{\theta}^{k}-\boldsymbol{\theta}^{*}\right\rangle.
\end{align}
Then we have
\begin{align}
&\mathbb{L}^{k}=f(\boldsymbol{\theta}^{k})-f\left(\boldsymbol{\theta}^{*}\right) + {\beta}_1 \left\|\boldsymbol{\theta}^{k}-\boldsymbol{\theta}^{k-1}\right\|^{2}\notag\\
&\qquad\qquad\quad+ {\beta}_2 \left\|\boldsymbol{\theta}^{k-1}-\boldsymbol{\theta}^{k-2}\right\|^{2}\notag\\
&\le \left\langle\nabla f(\boldsymbol{\theta}^{k}), \boldsymbol{\theta}^{k}-\boldsymbol{\theta}^{*}\right\rangle \notag\\
&\qquad + \left\langle \sqrt{\beta_1}\|\boldsymbol{\theta}^{k}-\boldsymbol{\theta}^{k-1}\|, \sqrt{\beta_1}\|\boldsymbol{\theta}^{k}-\boldsymbol{\theta}^{k-1}\| \right\rangle \notag\\
&\qquad+ \left\langle \sqrt{\beta_2}\|\boldsymbol{\theta}^{k-1}-\boldsymbol{\theta}^{k-2}\|, \sqrt{\beta_2}\|\boldsymbol{\theta}^{k-1}-\boldsymbol{\theta}^{k-2}\| \right\rangle\label{convexproperty1}\\
&=\Big\langle \left[\nabla f(\boldsymbol{\theta}^{k})^{\top},\sqrt{\beta_1}\|\boldsymbol{\theta}^{k}-\boldsymbol{\theta}^{k-1}\|, \sqrt{\beta_2}\|\boldsymbol{\theta}^{k-1}-\boldsymbol{\theta}^{k-2}\| \right]^{\top}, \notag\\
&\left[(\boldsymbol{\theta}^{k}-\boldsymbol{\theta}^{*})^{\top}, \sqrt{\beta_1}\|\boldsymbol{\theta}^{k}-\boldsymbol{\theta}^{k-1}\|, \sqrt{\beta_2}\|\boldsymbol{\theta}^{k-1}-\boldsymbol{\theta}^{k-2}\| \right]^{\top}\Big\rangle\label{innerproduct}\\
&\le\sqrt{\Big(\|\nabla f(\boldsymbol{\theta}^k)\|^2 + \beta_{1}\|\boldsymbol{\theta}^{k}-\boldsymbol{\theta}^{k-1}\|^{2} + \beta_{2}\|\boldsymbol{\theta}^{k-1}-\boldsymbol{\theta}^{k-2}\|^{2} \Big)}\notag\\
&\sqrt{\Big(\|\boldsymbol{\theta}^k-\boldsymbol{\theta}^*\|^2 + \beta_{1}\|\boldsymbol{\theta}^{k}-\boldsymbol{\theta}^{k-1}\|^{2} + \beta_{2}\|\boldsymbol{\theta}^{k-1}-\boldsymbol{\theta}^{k-2}\|^{2}\Big)}\label{finalpropert12}.
\end{align}
The result in (\ref{convexproperty1}) is obtained using (\ref{conveproer12}). In going from (\ref{innerproduct}) to (\ref{finalpropert12}), we employ the inequality $\langle\boldsymbol{x},\boldsymbol{y}\rangle\le\|\boldsymbol{x}\|\|\boldsymbol{y}\|$. The proof of Lemma \ref{lyapunvfunctinequa} is complete.
\end{proof}

Under {Assumption \ref{OHBassumption1}}, if constants $\alpha$ and $\xi$ are chosen properly so that (\ref{conefpositveLay})--(\ref{conefpositveLayend}) are satisfied with $\sigma_0\sigma_1\sigma_2\ne0$, then the result in (\ref{nonincreasing}) of Lemma \ref{descentLyapunovfunv} indicates that
\begin{align}
&\mathbb{L}^{k+1}-\mathbb{L}^{k} \notag\\ &\leq-\sigma_{0}\left\|\nabla f\left(\boldsymbol{\theta}^{k}\right)\right\|^{2}
-\sigma_{1}\left\|\boldsymbol{\theta}^{k}-\boldsymbol{\theta}^{k-1}\right\|^{2}\notag\\
&\qquad-\sigma_{2}\left\|\boldsymbol{\theta}^{k-1}-\boldsymbol{\theta}^{k-2}\right\|^{2}\\
&\le -\min\{\sigma_0,\frac{\sigma_1}{\beta_1},\frac{\sigma_2}{\beta_2}\}\Big(\left\|\nabla f\left(\boldsymbol{\theta}^{k}\right)\right\|^{2} + \beta_1\left\| \boldsymbol{\theta}^{k}-\boldsymbol{\theta}^{k-1} \right\|^{2} \notag\\
&\qquad+ \beta_2\left\| \boldsymbol{\theta}^{k-1}-\boldsymbol{\theta}^{k-2} \right\|^{2}\Big)\label{decreasLay112}
\end{align}
and $\min\{\sigma_0,\frac{\sigma_1}{\beta_1},\frac{\sigma_2}{\beta_2}\}>0$. The result in (\ref{decreasLay112})
implies
\begin{align}\label{boudnfunctiondiff}
&\left\|\nabla f\left(\boldsymbol{\theta}^{k}\right)\right\|^{2} + \beta_1\left\| \boldsymbol{\theta}^{k}-\boldsymbol{\theta}^{k-1} \right\|^{2} + \beta_2\left\| \boldsymbol{\theta}^{k-1}-\boldsymbol{\theta}^{k-2} \right\|^{2}\notag\\
&\qquad\le\frac{1}{\min\{\sigma_0,\sigma_1/\beta_1,\sigma_2/\beta_2\}}\left( -\mathbb{L}^{k+1}+\mathbb{L}^{k}\right).
\end{align}
Since the objective function $ f(\boldsymbol{\theta})$ is coercive in Assumption \ref{OHBassumption1}, we have $ f(\boldsymbol{\theta}^*)\le f(\boldsymbol{\theta}^k)<+\infty$ and $\|\boldsymbol{\theta}^k\|<+\infty$ which implies
$\left\|\boldsymbol{\theta}^{k}-\boldsymbol{\theta}^{*}\right\|<+\infty$, $\left\|\boldsymbol{\theta}^{k}-\boldsymbol{\theta}^{k-1}\right\|<+\infty$ and $\left\|\boldsymbol{\theta}^{k-1}-\boldsymbol{\theta}^{k-2}\right\|<+\infty$.
Then we obtain that there exists a finite number $N'<\infty$ such that
\begin{align}
\|\boldsymbol{\theta}^k-\boldsymbol{\theta}^*\|^2 + \beta_{1}\|\boldsymbol{\theta}^{k}-\boldsymbol{\theta}^{k-1}\|^{2} +\beta_{2}\|\boldsymbol{\theta}^{k-1}-\boldsymbol{\theta}^{k-2}\|^{2} \le N'. \label{boundfinite}
\end{align}
Plugging (\ref{boudnfunctiondiff}) and (\ref{boundfinite}) into (\ref{finalpropert12}) implies that
\begin{align}
&(\mathbb{L}^{k})^2\notag\\
&\le \Big(\|\nabla f(\boldsymbol{\theta}^k)\|^2 + \beta_{1}\|\boldsymbol{\theta}^{k}-\boldsymbol{\theta}^{k-1}\|^{2} + \beta_{2}\|\boldsymbol{\theta}^{k-1}-\boldsymbol{\theta}^{k-2}\|^{2} \Big)\notag\\
&\qquad\Big(\|\boldsymbol{\theta}^k-\boldsymbol{\theta}^*\|^2 + \beta_{1}\|\boldsymbol{\theta}^{k}-\boldsymbol{\theta}^{k-1}\|^{2} + \beta_{2}\|\boldsymbol{\theta}^{k-1}-\boldsymbol{\theta}^{k-2}\|^{2}\Big)\\
&\le\frac{N'}{\min\{\sigma_0,\sigma_1/\beta_1,\sigma_2/\beta_2\}}\left( \mathbb{L}^{k}-\mathbb{L}^{k+1}\right).\label{desceprotty}
\end{align}
Since Lemma \ref{descentLyapunovfunv} shows that $\mathbb{L}^{k+1}\le\mathbb{L}^{k}$, then we obtain $\mathbb{L}^{k}\mathbb{L}^{k+1}\le(\mathbb{L}^{k})^2$ and  the result in (\ref{desceprotty}) indicates that
\begin{align}
\mathbb{L}^{k}\mathbb{L}^{k+1}\le\frac{N'}{\min\{\sigma_0,\sigma_1/\beta_1,\sigma_2/\beta_2\}}\left( \mathbb{L}^{k}-\mathbb{L}^{k+1}\right)
\end{align}
which is equivalent to
\begin{align}\label{sumLay12}
\frac{1}{\mathbb{L}^{k+1}} - \frac{1}{\mathbb{L}^{k}} \ge\frac{\min\{\sigma_0,\sigma_1/\beta_1,\sigma_2/\beta_2\}}{N'}.
\end{align}
Using (\ref{sumLay12}) yields
\begin{align}\label{}
\frac{1}{\mathbb{L}^{k}}\ge\frac{1}{\mathbb{L}^{k}} - \frac{1}{\mathbb{L}^{0}} \ge\frac{k\min\{\sigma_0,\sigma_1/\beta_1,\sigma_2/\beta_2\}}{N'}.
\end{align}
This completes the proof of Theorem \ref{convextheo}.

\end{proof}

\section{Proof of Theorem \ref{convergenceTheorem3}}\label{convergenonconv}
The derivation in this part is similar to that in \cite{chen2018lag} that does not consider adaptive sparsification and error correction.
The result in Lemma \ref{descentLyapunovfunv} shows that
\begin{align}
&\mathbb{L}^{k+1}-\mathbb{L}^{k} \notag\\ &\leq-\sigma_{0}\left\|\nabla f\left(\boldsymbol{\theta}^{k}\right)\right\|^{2}
-\sigma_{1}\left\|\boldsymbol{\theta}^{k}-\boldsymbol{\theta}^{k-1}\right\|^{2}\notag\\
&\qquad\qquad-\sigma_{2}\left\|\boldsymbol{\theta}^{k-1}-\boldsymbol{\theta}^{k-2}\right\|^{2}\\
&\le -\min\{\sigma_0,\frac{\sigma_1}{\beta_1},\frac{\sigma_2}{\beta_2}\}\Big(\left\|\nabla f\left(\boldsymbol{\theta}^{k}\right)\right\|^{2} + \beta_1\left\| \boldsymbol{\theta}^{k}-\boldsymbol{\theta}^{k-1} \right\|^{2}\notag\\
&\qquad + \beta_2\left\| \boldsymbol{\theta}^{k-1}-\boldsymbol{\theta}^{k-2} \right\|^{2}\Big).\label{decreasLay11}
\end{align}
Then summing over $k'\in\{1,2,...,k\}$, we have
\begin{align}
&\mathbb{L}^{1} - \mathbb{L}^{k+1}  \notag\\ &\ge\min\{\sigma_0,\frac{\sigma_1}{\beta_1},\frac{\sigma_2}{\beta_2}\}\sum_{k'=1}^k\Big(\left\|\nabla f\left(\boldsymbol{\theta}^{k'}\right)\right\|^{2}
 + \beta_1\left\| \boldsymbol{\theta}^{k'}-\boldsymbol{\theta}^{k'-1} \right\|^{2}\notag\\
&\qquad+ \beta_2\left\| \boldsymbol{\theta}^{k'-1}-\boldsymbol{\theta}^{k'-2} \right\|^{2}\Big).
\end{align}
Since we have $\mathbb{L}^{1} - \mathbb{L}^{k+1}\le\mathbb{L}^{1}<\infty$, then we obtain (with $\min\{\sigma_0,{\sigma_1}/{\beta_1},{\sigma_2}/{\beta_2}\}>0$)
\begin{align}
&\lim_{k\rightarrow\infty}\sum_{k'=1}^k\Big(\left\|\nabla f\left(\boldsymbol{\theta}^{k'}\right)\right\|^{2} + \beta_1\left\| \boldsymbol{\theta}^{k'}-\boldsymbol{\theta}^{k'-1} \right\|^{2} \notag\\
&\qquad+ \beta_2\left\| \boldsymbol{\theta}^{k'-1}-\boldsymbol{\theta}^{k'-2} \right\|^{2}\Big)<\infty
\end{align}
which implies
\begin{align}
\lim_{k\rightarrow\infty}\min_{1\le k'\le k}\left\|\nabla f\left(\boldsymbol{\theta}^{k'}\right)\right\|^{2}\rightarrow 0.
\end{align}
\textcolor[rgb]{0.00,0.00,0.00}{Additionally, according to the lemma on the convergence rates of nonnegative summable sequences \cite{chen2018lag,davis2016convergence}, we can also obtain
\begin{align}
\min_{1\le k'\le k}\left\|\nabla f(\boldsymbol{\theta}^{k'})\right\|^{2}=\mathcal{O}(1 / k).
\end{align}} This completes the proof of Theorem \ref{convergenceTheorem3}.

\ifCLASSOPTIONcaptionsoff
  \newpage
\fi



%
\input{refs.bbl}

\bibliographystyle{IEEEtran}
\bibliography{refs}

%
%

%

%
%
%




\end{document}

%% file: refs.bbl

%% file: sparsification_paper_revision.bbl
\begin{thebibliography}{10}
\providecommand{\url}[1]{#1}
\csname url@samestyle\endcsname
\providecommand{\newblock}{\relax}
\providecommand{\bibinfo}[2]{#2}
\providecommand{\BIBentrySTDinterwordspacing}{\spaceskip=0pt\relax}
\providecommand{\BIBentryALTinterwordstretchfactor}{4}
\providecommand{\BIBentryALTinterwordspacing}{\spaceskip=\fontdimen2\font plus
\BIBentryALTinterwordstretchfactor\fontdimen3\font minus
  \fontdimen4\font\relax}
\providecommand{\BIBforeignlanguage}[2]{{%
\expandafter\ifx\csname l@#1\endcsname\relax
\typeout{** WARNING: IEEEtran.bst: No hyphenation pattern has been}%
\typeout{** loaded for the language `#1'. Using the pattern for}%
\typeout{** the default language instead.}%
\else
\language=\csname l@#1\endcsname
\fi
#2}}
\providecommand{\BIBdecl}{\relax}
\BIBdecl

\bibitem{song2021federated}
Z.~Song, H.~Sun, H.~H. Yang, X.~Wang, and T.~Q. Quek, ``Federated learning in
  multi-antenna wireless networks,'' in \emph{2021 IEEE International
  Conference on Communications Workshops (ICC Workshops)}.\hskip 1em plus 0.5em
  minus 0.4em\relax IEEE, 2021, pp. 1--6.

\bibitem{chen2021fedsvrg}
D.~Chen, C.~S. Hong, Y.~Zha, Y.~Zhang, X.~Liu, and Z.~Han, ``{FedSVRG} based
  communication efficient scheme for federated learning in {MEC} networks,''
  \emph{IEEE Transactions on Vehicular Technology}, vol.~70, no.~7, pp.
  7300--7304, 2021.

\bibitem{bedi2019asynchronous}
A.~S. Bedi, A.~Koppel, and K.~Rajawat, ``Asynchronous online learning in
  multi-agent systems with proximity constraints,'' \emph{IEEE Transactions on
  Signal and Information Processing over Networks}, vol.~5, no.~3, pp.
  479--494, 2019.

\bibitem{nedic2009distributed}
A.~Nedic and A.~Ozdaglar, ``Distributed subgradient methods for multi-agent
  optimization,'' \emph{IEEE Transactions on Automatic Control}, vol.~54,
  no.~1, pp. 48--61, 2009.

\bibitem{bullo2009distributed}
F.~Bullo, J.~Cortes, and S.~Martinez, \emph{Distributed control of robotic
  networks: a mathematical approach to motion coordination algorithms}.\hskip
  1em plus 0.5em minus 0.4em\relax Princeton University Press, 2009.

\bibitem{cao2012overview}
Y.~Cao, W.~Yu, W.~Ren, and G.~Chen, ``An overview of recent progress in the
  study of distributed multi-agent coordination,'' \emph{IEEE Transactions on
  Industrial informatics}, vol.~9, no.~1, pp. 427--438, 2012.

\bibitem{rabbat2004distributed}
M.~Rabbat and R.~Nowak, ``Distributed optimization in sensor networks,'' in
  \emph{Proceedings of the 3rd international symposium on Information
  processing in sensor networks}.\hskip 1em plus 0.5em minus 0.4em\relax ACM,
  2004, pp. 20--27.

\bibitem{liu2017distributed}
H.~J. Liu, W.~Shi, and H.~Zhu, ``Distributed voltage control in distribution
  networks: Online and robust implementations,'' \emph{IEEE Transactions on
  Smart Grid}, vol.~9, no.~6, pp. 6106--6117, 2017.

\bibitem{lin2017deep}
Y.~Lin, S.~Han, H.~Mao, Y.~Wang, and W.~J. Dally, ``Deep gradient compression:
  Reducing the communication bandwidth for distributed training,'' \emph{arXiv
  preprint arXiv:1712.01887}, 2017.

\bibitem{konevcny2016federated}
J.~Kone\v{c}n\'y, H.~B. McMahan, F.~X. Yu, P.~Richt{\'a}rik, A.~T. Suresh, and
  D.~Bacon, ``Federated learning: Strategies for improving communication
  efficiency,'' \emph{arXiv preprint arXiv:1610.05492}, 2016.

\bibitem{lian2017can}
X.~Lian, C.~Zhang, H.~Zhang, C.-J. Hsieh, W.~Zhang, and J.~Liu, ``Can
  decentralized algorithms outperform centralized algorithms? a case study for
  decentralized parallel stochastic gradient descent,'' in \emph{Advances in
  Neural Information Processing Systems}, 2017, pp. 5330--5340.

\bibitem{alistarh2018convergence}
D.~Alistarh, T.~Hoefler, M.~Johansson, N.~Konstantinov, S.~Khirirat, and
  C.~Renggli, ``The convergence of sparsified gradient methods,'' in
  \emph{Advances in Neural Information Processing Systems}, 2018, pp.
  5973--5983.

\bibitem{bekkerman2011scaling}
R.~Bekkerman, M.~Bilenko, and J.~Langford, \emph{Scaling up machine learning:
  Parallel and distributed approaches}.\hskip 1em plus 0.5em minus 0.4em\relax
  Cambridge University Press, 2011.

\bibitem{shalf2010exascale}
J.~Shalf, S.~Dosanjh, and J.~Morrison, ``Exascale computing technology
  challenges,'' in \emph{International Conference on High Performance Computing
  for Computational Science}.\hskip 1em plus 0.5em minus 0.4em\relax Springer,
  2010, pp. 1--25.

\bibitem{mcmahan2016communication}
H.~B. McMahan, E.~Moore, D.~Ramage, S.~Hampson \emph{et~al.},
  ``Communication-efficient learning of deep networks from decentralized
  data,'' \emph{arXiv preprint arXiv:1602.05629}, 2016.

\bibitem{smith2017federated}
V.~Smith, C.-K. Chiang, M.~Sanjabi, and A.~S. Talwalkar, ``Federated multi-task
  learning,'' in \emph{Advances in Neural Information Processing Systems},
  2017, pp. 4424--4434.

\bibitem{stoica2017berkeley}
I.~Stoica, D.~Song, R.~A. Popa, D.~Patterson, M.~W. Mahoney, R.~Katz, A.~D.
  Joseph, M.~Jordan, J.~M. Hellerstein, J.~E. Gonzalez \emph{et~al.}, ``A
  berkeley view of systems challenges for ai,'' \emph{arXiv preprint
  arXiv:1712.05855}, 2017.

\bibitem{nguyen2017sarah}
L.~M. Nguyen, J.~Liu, K.~Scheinberg, and M.~Tak{\'a}{\v{c}}, ``Sarah: A novel
  method for machine learning problems using stochastic recursive gradient,''
  in \emph{International Conference on Machine Learning}.\hskip 1em plus 0.5em
  minus 0.4em\relax PMLR, 2017, pp. 2613--2621.

\bibitem{roux2012stochastic}
N.~L. Roux, M.~Schmidt, and F.~R. Bach, ``A stochastic gradient method with an
  exponential convergence \_rate for finite training sets,'' in \emph{Advances
  in neural information processing systems}, 2012, pp. 2663--2671.

\bibitem{defazio2014saga}
A.~Defazio, F.~Bach, and S.~Lacoste-Julien, ``Saga: A fast incremental gradient
  method with support for non-strongly convex composite objectives,'' in
  \emph{Advances in neural information processing systems}, 2014, pp.
  1646--1654.

\bibitem{johnson2013accelerating}
R.~Johnson and T.~Zhang, ``Accelerating stochastic gradient descent using
  predictive variance reduction,'' in \emph{Advances in neural information
  processing systems}, 2013, pp. 315--323.

\bibitem{xiao2014proximal}
L.~Xiao and T.~Zhang, ``A proximal stochastic gradient method with progressive
  variance reduction,'' \emph{SIAM Journal on Optimization}, vol.~24, no.~4,
  pp. 2057--2075, 2014.

\bibitem{liu2019communication}
Y.~Liu, W.~Xu, G.~Wu, Z.~Tian, and Q.~Ling, ``Communication-censored admm for
  decentralized consensus optimization,'' \emph{IEEE Transactions on Signal
  Processing}, vol.~67, no.~10, pp. 2565--2579, 2019.

\bibitem{boyd2011distributed}
S.~Boyd, N.~Parikh, and E.~Chu, \emph{Distributed optimization and statistical
  learning via the alternating direction method of multipliers}.\hskip 1em plus
  0.5em minus 0.4em\relax Now Publishers Inc, 2011.

\bibitem{zhou2021communication}
S.~Zhou and G.~Y. Li, ``Communication-efficient admm-based federated
  learning,'' \emph{arXiv preprint arXiv:2110.15318}, 2021.

\bibitem{chang2014multi}
T.-H. Chang, M.~Hong, and X.~Wang, ``Multi-agent distributed optimization via
  inexact consensus admm,'' \emph{IEEE Transactions on Signal Processing},
  vol.~63, no.~2, pp. 482--497, 2014.

\bibitem{agarwal2021utility}
S.~Agarwal, H.~Wang, S.~Venkataraman, and D.~Papailiopoulos, ``On the utility
  of gradient compression in distributed training systems,'' \emph{arXiv
  preprint arXiv:2103.00543}, 2021.

\bibitem{seide20141}
F.~Seide, H.~Fu, J.~Droppo, G.~Li, and D.~Yu, ``1-bit stochastic gradient
  descent and its application to data-parallel distributed training of speech
  dnns,'' in \emph{Fifteenth Annual Conference of the International Speech
  Communication Association}, 2014.

\bibitem{zhang2018distributed}
J.~Zhang, K.~You, and T.~Ba{\c{s}}ar, ``Distributed discrete-time optimization
  in multiagent networks using only sign of relative state,'' \emph{IEEE
  Transactions on Automatic Control}, vol.~64, no.~6, pp. 2352--2367, 2018.

\bibitem{alistarh2017qsgd}
D.~Alistarh, D.~Grubic, J.~Li, R.~Tomioka, and M.~Vojnovic, ``Qsgd:
  Communication-efficient sgd via gradient quantization and encoding,'' in
  \emph{Advances in Neural Information Processing Systems}, 2017, pp.
  1709--1720.

\bibitem{horvath2019stochastic}
S.~Horv{\'a}th, D.~Kovalev, K.~Mishchenko, S.~Stich, and P.~Richt{\'a}rik,
  ``Stochastic distributed learning with gradient quantization and variance
  reduction,'' \emph{arXiv preprint arXiv:1904.05115}, 2019.

\bibitem{reisizadeh2019exact}
A.~Reisizadeh, A.~Mokhtari, H.~Hassani, and R.~Pedarsani, ``An exact quantized
  decentralized gradient descent algorithm,'' \emph{IEEE Transactions on Signal
  Processing}, vol.~67, no.~19, pp. 4934--4947, 2019.

\bibitem{elgabli2019q}
A.~Elgabli, J.~Park, A.~S. Bedi, M.~Bennis, and V.~Aggarwal, ``Q-gadmm:
  Quantized group admm for communication efficient decentralized machine
  learning,'' \emph{arXiv preprint arXiv:1910.10453}, 2019.

\bibitem{aji2017sparse}
A.~F. Aji and K.~Heafield, ``Sparse communication for distributed gradient
  descent,'' \emph{arXiv preprint arXiv:1704.05021}, 2017.

\bibitem{stich2018sparsified}
S.~U. Stich, J.-B. Cordonnier, and M.~Jaggi, ``Sparsified sgd with memory,'' in
  \emph{Advances in Neural Information Processing Systems}, 2018, pp.
  4447--4458.

\bibitem{bryan2013making}
K.~Bryan and T.~Leise, ``Making do with less: an introduction to compressed
  sensing,'' \emph{Siam Review}, vol.~55, no.~3, pp. 547--566, 2013.

\bibitem{donoho2006compressed}
D.~L. Donoho, ``Compressed sensing,'' \emph{IEEE Transactions on information
  theory}, vol.~52, no.~4, pp. 1289--1306, 2006.

\bibitem{boas2016shrinkage}
T.~Boas, A.~Dutta, X.~Li, K.~P. Mercier, and E.~Niderman, ``Shrinkage function
  and its applications in matrix approximation,'' \emph{arXiv preprint
  arXiv:1601.07600}, 2016.

\bibitem{stewart1993early}
G.~W. Stewart, ``On the early history of the singular value decomposition,''
  \emph{SIAM review}, vol.~35, no.~4, pp. 551--566, 1993.

\bibitem{dutta2020discrepancy}
A.~Dutta, E.~H. Bergou, A.~M. Abdelmoniem, C.-Y. Ho, A.~N. Sahu, M.~Canini, and
  P.~Kalnis, ``On the discrepancy between the theoretical analysis and
  practical implementations of compressed communication for distributed deep
  learning,'' in \emph{Proceedings of the AAAI Conference on Artificial
  Intelligence}, vol.~34, no.~04, 2020, pp. 3817--3824.

\bibitem{strom2015scalable}
N.~Strom, ``Scalable distributed dnn training using commodity gpu cloud
  computing,'' in \emph{Sixteenth Annual Conference of the International Speech
  Communication Association}, 2015.

\bibitem{xu2021grace}
H.~Xu, C.-Y. Ho, A.~M. Abdelmoniem, A.~Dutta, E.~H. Bergou, K.~Karatsenidis,
  M.~Canini, and P.~Kalnis, ``Grace: A compressed communication framework for
  distributed machine learning,'' in \emph{Proc. of 41st IEEE Int. Conf.
  Distributed Computing Systems (ICDCS)}, 2021.

\bibitem{basu2020qsparse}
D.~Basu, D.~Data, C.~Karakus, and S.~N. Diggavi, ``Qsparse-local-sgd:
  Distributed sgd with quantization, sparsification, and local computations,''
  \emph{IEEE Journal on Selected Areas in Information Theory}, vol.~1, no.~1,
  pp. 217--226, 2020.

\bibitem{rago1996censoring}
C.~Rago, P.~Willett, and Y.~Bar-Shalom, ``Censoring sensors: A
  low-communication-rate scheme for distributed detection,'' \emph{IEEE
  Transactions on Aerospace and Electronic Systems}, vol.~32, no.~2, pp.
  554--568, 1996.

\bibitem{appadwedula2005energy}
S.~Appadwedula, V.~V. Veeravalli, and D.~L. Jones, ``Energy-efficient detection
  in sensor networks,'' \emph{IEEE Journal on Selected areas in
  communications}, vol.~23, no.~4, pp. 693--702, 2005.

\bibitem{marano2006cross}
S.~Marano, V.~Matta, P.~Willett, and L.~Tong, ``Cross-layer design of
  sequential detectors in sensor networks,'' \emph{IEEE Transactions on Signal
  Processing}, vol.~54, no.~11, pp. 4105--4117, 2006.

\bibitem{patwari2003hierarchical}
N.~Patwari, A.~Hero, and B.~M. Sadler, ``Hierarchical censoring sensors for
  change detection,'' in \emph{IEEE Workshop on Statistical Signal Processing,
  2003}.\hskip 1em plus 0.5em minus 0.4em\relax IEEE, 2003, pp. 21--24.

\bibitem{chen2018lag}
T.~Chen, G.~Giannakis, T.~Sun, and W.~Yin, ``Lag: Lazily aggregated gradient
  for communication-efficient distributed learning,'' in \emph{Advances in
  Neural Information Processing Systems}, 2018, pp. 5050--5060.

\bibitem{kairouz2019advances}
P.~Kairouz, H.~B. McMahan, B.~Avent, A.~Bellet, M.~Bennis, A.~N. Bhagoji,
  K.~Bonawitz, Z.~Charles, G.~Cormode, R.~Cummings \emph{et~al.}, ``Advances
  and open problems in federated learning,'' \emph{arXiv preprint
  arXiv:1912.04977}, 2019.

\bibitem{bonawitz2019towards}
K.~Bonawitz, H.~Eichner, W.~Grieskamp, D.~Huba, A.~Ingerman, V.~Ivanov,
  C.~Kiddon, J.~Kone{\v{c}}n{\`y}, S.~Mazzocchi, H.~B. McMahan \emph{et~al.},
  ``Towards federated learning at scale: System design,'' \emph{arXiv preprint
  arXiv:1902.01046}, 2019.

\bibitem{sergeev2018horovod}
A.~Sergeev and M.~Del~Balso, ``Horovod: fast and easy distributed deep learning
  in tensorflow,'' \emph{arXiv preprint arXiv:1802.05799}, 2018.

\bibitem{peressini1988mathematics}
A.~L. Peressini, F.~E. Sullivan, and J.~J. Uhl, \emph{The mathematics of
  nonlinear programming}.\hskip 1em plus 0.5em minus 0.4em\relax
  Springer-Verlag New York, 1988.

\bibitem{nesterov2013introductory}
Y.~Nesterov, \emph{Introductory lectures on convex optimization: A basic
  course}.\hskip 1em plus 0.5em minus 0.4em\relax Springer Science \& Business
  Media, 2013, vol.~87.

\bibitem{taylor2018lyapunov}
A.~Taylor, B.~Van~Scoy, and L.~Lessard, ``Lyapunov functions for first-order
  methods: Tight automated convergence guarantees,'' in \emph{International
  Conference on Machine Learning}.\hskip 1em plus 0.5em minus 0.4em\relax PMLR,
  2018, pp. 4897--4906.

\bibitem{brooks2020run}
J.~S. Brooks, R.~Golla, A.~Danysh, S.~Chavan, P.~Agrawal, A.~Ewoldt, and
  D.~Weaver, ``Run-length encoding decompression,'' Jan.~14 2020, uS Patent
  10,534,606.

\bibitem{reisizadeh2020fedpaq}
A.~Reisizadeh, A.~Mokhtari, H.~Hassani, A.~Jadbabaie, and R.~Pedarsani,
  ``Fedpaq: A communication-efficient federated learning method with periodic
  averaging and quantization,'' in \emph{International Conference on Artificial
  Intelligence and Statistics}.\hskip 1em plus 0.5em minus 0.4em\relax PMLR,
  2020, pp. 2021--2031.

\bibitem{schmidt2017minimizing}
M.~Schmidt, N.~Le~Roux, and F.~Bach, ``Minimizing finite sums with the
  stochastic average gradient,'' \emph{Mathematical Programming}, vol. 162, no.
  1-2, pp. 83--112, 2017.

\bibitem{xu2020second}
P.~Xu, F.~Roosta, and M.~W. Mahoney, ``Second-order optimization for non-convex
  machine learning: An empirical study,'' in \emph{Proceedings of the 2020 SIAM
  International Conference on Data Mining}.\hskip 1em plus 0.5em minus
  0.4em\relax SIAM, 2020, pp. 199--207.

\bibitem{lecun1998gradient}
Y.~LeCun, L.~Bottou, Y.~Bengio, and P.~Haffner, ``Gradient-based learning
  applied to document recognition,'' \emph{Proceedings of the IEEE}, vol.~86,
  no.~11, pp. 2278--2324, 1998.

\bibitem{chang2011libsvm}
C.-C. Chang and C.-J. Lin, ``Libsvm: A library for support vector machines,''
  \emph{ACM transactions on intelligent systems and technology (TIST)}, vol.~2,
  no.~3, pp. 1--27, 2011.

\bibitem{lewis2004rcv1}
D.~D. Lewis, Y.~Yang, T.~G. Rose, and F.~Li, ``Rcv1: A new benchmark collection
  for text categorization research,'' \emph{Journal of machine learning
  research}, vol.~5, no. Apr, pp. 361--397, 2004.

\bibitem{yang2019scheduling}
H.~H. Yang, Z.~Liu, T.~Q. Quek, and H.~V. Poor, ``Scheduling policies for
  federated learning in wireless networks,'' \emph{IEEE Transactions on
  Communications}, vol.~68, no.~1, pp. 317--333, 2019.

\bibitem{davis2016convergence}
D.~Davis and W.~Yin, ``Convergence rate analysis of several splitting
  schemes,'' in \emph{Splitting methods in communication, imaging, science, and
  engineering}.\hskip 1em plus 0.5em minus 0.4em\relax Springer, 2016, pp.
  115--163.

\end{thebibliography}
